\documentclass{article}

\usepackage[preprint,nonatbib]{neurips_2020}    %
\usepackage[utf8]{inputenc} %
\usepackage[T1]{fontenc}    %
\usepackage{tabularx}
\usepackage{hyperref}       %
\usepackage{url}            %
\usepackage{placeins}
\usepackage{booktabs}       %
\usepackage{amsfonts}       %
\usepackage{nicefrac}       %
\usepackage{microtype}      %
\usepackage[numbers]{natbib}
\usepackage{proba}
\usepackage{wrapfig}
\usepackage{amssymb}%
\usepackage{pifont}%
\usepackage{subcaption}
\usepackage{rotating}
\usepackage{etoolbox}
\usepackage{xparse}
\usepackage{grffile}
\usepackage{amsmath}
\usepackage{xspace}
\usepackage{enumitem}
\usepackage{euscript}
\usepackage{mathtools}
\usepackage{tikz}
\usepackage{enumitem}
\usepackage{mathtools}
\usepackage{tablefootnote}
\usepackage[flushleft]{threeparttable}
\usepackage{multirow}
\usepackage[title]{appendix}
\usepackage{arydshln}
\usepackage{array, booktabs}%
\usepackage[nameinlink]{cleveref}
\usepackage{comment}
\usepackage[utf8]{inputenc}
\usepackage{url}
\usepackage{graphicx}
\usepackage{amssymb}
\usepackage{dsfont}
\usepackage{amsmath,amsthm}
\usepackage{mathtools}
\usepackage{booktabs}
\usepackage{balance}
\usepackage{xspace}
\usepackage{multirow}
\usepackage{multicol}
\usepackage{xcolor}
\usepackage{comment}
\usepackage{nameref}
\usepackage{pdfpages}
\usepackage{enumitem}
\usepackage{tikz}

\usepackage[normalem]{ulem}

\newcommand*\dbar[1]{\overline{\overline{\lower0.2ex\hbox{$#1$}}}}

\usetikzlibrary{shapes.misc}
\colorlet{darkgreen}{green!50!black}

\NewDocumentCommand{\Exp}{d() r[]}{\ensuremath{
		\mathds{E}
		\IfValueT{#1}{_{#1}}
		{\left[#2\right]}
}}

\newtheorem{theorem}{Theorem}

\newtheorem{definition}{Definition}

\newtheorem{lemma}{Lemma}

\newcommand{\method}{MHM-GNN\xspace}
\newcommand{\Appendix}{supplement\xspace}
\newcommand{\highorder}{joint $k$-node\xspace}

\newcommand{\eat}[1]{}

\newcommand{\cH}{\mathcal{H}}

\newcommand{\cC}{\mathcal{C}}

\newcommand{\cD}{\mathcal{D}}
\newcommand{\cT}{\mathcal{T}}
\newcommand{\cP}{\mathcal{P}}

\newcommand{\cI}{\mathcal{I}}
\newcommand{\cL}{\mathcal{L}}

\newcommand{\pr}{\mathbb{P}}

\newcommand{\bA}{{\bf A}}

\newcommand{\bY}{{\bf Y}}
\newcommand{\by}{{\bf y}}
\newcommand{\bW}{{\bf W}}

\definecolor{pastelblue}{RGB}{76,113,175}
\definecolor{pastelgreen}{RGB}{84,167,104}
\definecolor{pastelred}{RGB}{196,78,82}
\definecolor{pastelgrey}{RGB}{230,230,230}
\definecolor{pastelbeige}{RGB}{243,236,221}
\definecolor{pastelpurple}{RGB}{154,139,192}
\definecolor{mysalmon}{RGB}{250, 128, 114}
\definecolor{myblue}{RGB}{60,105,210}
\definecolor{mygreen}{RGB}{60,179,113}

\usepackage{amsmath,amsfonts,bm}
\def\eqref#1{equation~\ref{#1}}

\def\1{\bm{1}}

\def\mA{{\bm{A}}}

\def\mM{{\bm{M}}}

\def\mT{{\bm{T}}}

\def\mX{{\bm{X}}}

\DeclareMathAlphabet{\mathsfit}{\encodingdefault}{\sfdefault}{m}{sl}
\SetMathAlphabet{\mathsfit}{bold}{\encodingdefault}{\sfdefault}{bx}{n}

\def\cC{{\mathcal{C}}}
\def\cD{{\mathcal{D}}}

\def\cH{{\mathcal{H}}}
\def\cI{{\mathcal{I}}}

\def\cL{{\mathcal{L}}}

\def\cP{{\mathcal{P}}}

\def\cT{{\mathcal{T}}}

\usepackage[utf8]{inputenc} %
\usepackage[T1]{fontenc}    %
\usepackage{hyperref}       %
\usepackage{url}            %
\usepackage{booktabs}       %
\usepackage{amsfonts}       %
\usepackage{nicefrac}       %
\usepackage{microtype}      %

\title{Unsupervised Joint $k$-node Graph Representations with Compositional Energy-Based Models}

\author{%
  Leonardo Cotta\thanks{\url{http://cottascience.github.io/}} \\
  Purdue University\\
  \texttt{cotta@purdue.edu}
  \And
  Carlos H. C. Teixeira \\
  Universidade Federal de Minas Gerais, Brazil \\
  \texttt{carlos@dcc.ufmg.br}
  \And
  Ananthram Swami \\
  United States Army Research Laboratory \\
  \texttt{ananthram.swami.civ@mail.mil}
  \And
  Bruno Ribeiro\\
  Purdue University\\
  \texttt{ribeiro@cs.purdue.edu}
}

\begin{document}

\maketitle

\begin{abstract}
  Existing Graph Neural Network (GNN) methods that learn \textit{inductive unsupervised} graph representations focus on learning node and edge representations by predicting observed edges in the graph. Although such approaches have shown advances in downstream node classification tasks, they are ineffective in jointly representing larger $k$-node sets, $k{>}2$. 
  We propose \method, an inductive unsupervised graph representation approach that combines joint $k$-node representations with energy-based models (hypergraph Markov networks) and GNNs. 
  To address the intractability of the loss that arises from this combination, we  endow our optimization with a loss upper bound using a finite-sample unbiased Markov Chain Monte Carlo estimator. 
   Our experiments show that the unsupervised \highorder representations of \method produce better unsupervised representations than existing approaches from the literature. 
\end{abstract}

\section{Introduction}

Inductive unsupervised learning using Graph Neural Networks (GNNs) in (dyadic) graphs is currently restricted to node and edge representations due to their reliance on edge-based losses~\citep{bojchevski2018deep,Hamilton2017,kipf2016variational,velickovic2018graph}. 
If we want to tackle downstream tasks that require jointly reasoning about $k>2$ nodes, but whose input data are dyadic relations (i.e., standard graphs) rather than hyperedges, we must develop techniques that can go beyond edge-based losses.

Joint $k$-node representation tasks with dyadic relational inputs include drone swarms that communicate amongst themselves to jointly act on a task~\cite{shi2020neural,taylor2019learning}, but also include more traditional product-recommendation tasks. For instance, an e-commerce website might want to predict which $k$ products could be jointly purchased in the same shopping cart, while the database only records (product, product) dyads to safeguard user information.

\citet{srinivasan2020equivalence} have recently shown that GNN node representations are insufficient to capture joint characteristics of $k$ nodes that are unique to this group of nodes. 
Indeed, our experiments show that using existing unsupervised GNN ---with their node representations and edge losses--- one cannot accurately detect these $k$-product carts on an e-commerce website. 
Unfortunately, existing GNN extensions that give \highorder representations require {\em supervised} graph-wide losses~\citep{morris2019weisfeiler,maronBSL19}, leaving a significant gap between edge and {\em supervised} whole-graph losses (i.e., we need multiple labeled graphs for these to work). The main reason for this gap is scalability: to obtain {\em true} {\em unsupervised} \highorder representations, one must optimize a model defined over {\em all} $k$-node induced subgraphs of a graph. 

Our approach  \method (Motif Hypergraph Markov Graph Neural Networks) leverages the compositionality of hypergraph Markov network models (HMNs)~\citep{rowland2017uprooting,zheleva2010higher,kohli2009robust} that allows us to define an unsupervised objective (energy-based model) over GNN representations of motifs (see upper half of \Cref{fig:diagram}).

\begin{figure*}[t!!]
	\vspace{-.18in}
	\centering
	\includegraphics[height=1.8in, width=1\linewidth]{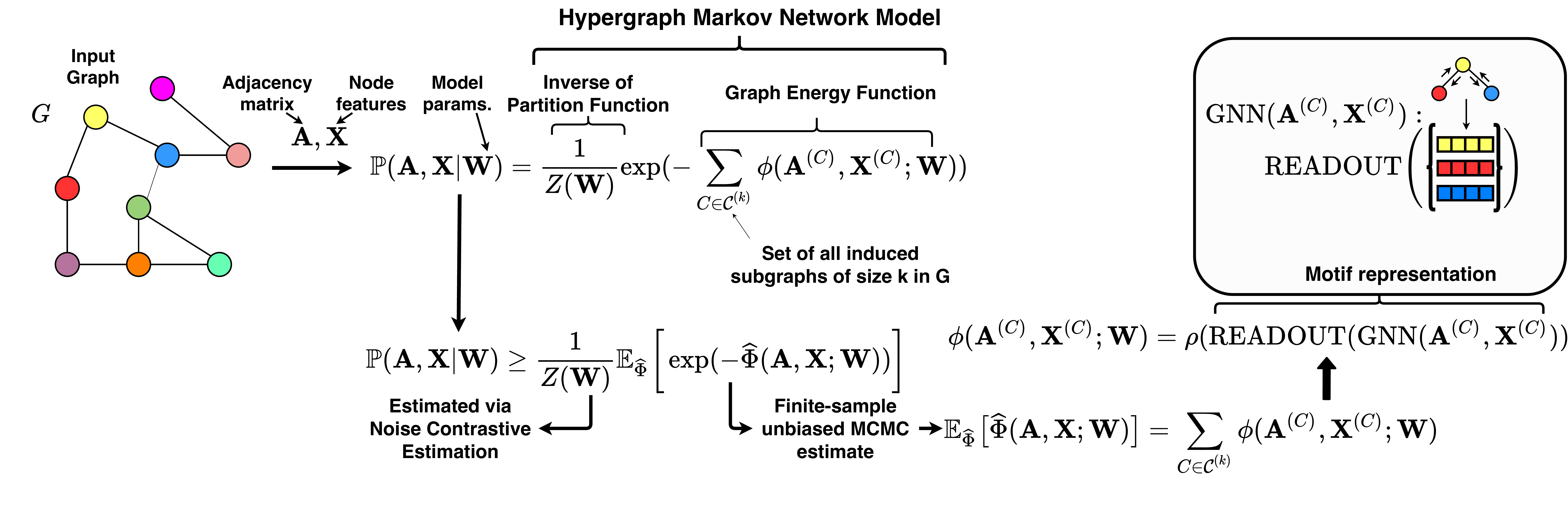}
	\vspace{-0.3in}
	\caption{The proposed unsupervised graph representation using motif compositions. Here, we present the \method model from \Cref{eq:model}, the energy estimator $\widehat{\Phi}$ from \Cref{eq:estimator}, the motif energy and representation from \Cref{eq:energy}.}\label{fig:diagram}
	\vspace{-0.06in}
\end{figure*}

Scalability is the main challenge we have to overcome, a type of scalability issue not addressed in the hypergraph Markov network literature~\citep{rowland2017uprooting,zheleva2010higher,kohli2009robust}.
First, there is the traditional likelihood intractability associated with computing the partition function $Z(\bW)$ of energy models ---$Z(\bW)$ is shown in the likelihood $\pr(\mA,\mX|\bW)$ in \Cref{fig:diagram} and also in \Cref{eq:model}.
There are standard solutions for this challenge (e.g., Noise-Contrastive Estimation (NCE)~\citep{gutmann2010noise}).
The more vexing challenge comes from the intractability created by our inductive graph representation that applies motif representations to all $k$-node subgraphs, which requires $\binom{n}{k}$ operations per gradient step, typically with $n \gg k$.
To make this step tractable, we leverage recent advances in finite-sample unbiased Markov Chain Monte Carlo estimation for sums of subgraph functions over large graphs~\citep{teixeira2018graph}.
This unbiased estimate, combined with Jensen's inequality, allows us to optimize a  lower bound on the intractable likelihood (assuming $Z(\bW)$ is known).
Fold that into the asymptotics of NCE and we get a principled, tractable optimization.

{\bf Contributions.} Our contributions are three-fold. First, we introduce \method, which produces joint $(k>2)$-node representations, where $k$ is a hyperparameter of the model. Second, we introduce a principled and scalable stochastic optimization method that learns
\method with a finite-sample unbiased estimator of the graph energy (see Fig.~\ref{fig:diagram}) and a NCE objective. Finally, we show how the joint $k$-node representations from \method produce better unsupervised \highorder representations than existing approaches that aggregate node representations.

\section{Related Work}\label{sec:rw}

In this section, we briefly review existing approaches to \textit{inductive unsupervised} representation learning of graphs, discuss existing work with higher-order graph representations and overview energy-based models. Finally, we present what in literature is not related to this work.

{\bf Edge-based graph models.} Although graph models are prominent in many areas of research~\citep{newman2018networks}, most of the proposed models, such as the initial Erd\"{o}s-Rényi model~\citep{erdos59a}, stochastic block models~\citep{holland1983stochastic} and the more recent neural network-based approaches~\citep{kipf2016variational, Hamilton2017,bojchevski2018deep} assume conditional independence of edges, resulting in what is often called an edge-based loss function. That is, all such models assume the appearance of edges in the graph is independent given the edge representations, which is usually computed via their endpoints' representations. This important conditional independence assumption appears in what we call edge-based graph models. There exist alternatives such as Markov Random Graphs~\citep{frank1986markov}, where an edge is dependent on every other edge that shares one of its endpoints, but graph models without any conditional independence assumption are still not commonly used.

{\bf Inductive unsupervised node representations with GNNs.} Recently, GraphSAGE~\citep{Hamilton2017} introduced the use of GNNs to learn inductive node representations in an unsupervised manner by applying an edge-based loss while using short random walks. There are also auto-encoder approaches~\citep{kipf2016variational,pan2018adversarially,samanta2018designing}, where one tries to reconstruct the edges in the graph using node representations. Auto-encoders also assume conditional independence of edges and can be classified as edge-based models. In contrast to edge-based loss models, DGI~\citep{velickovic2018graph} minimizes the mutual entropy between node representations and a whole-graph representation ---it does not model a probability distribution. Whilst the combination of GNNs and edge-based models has been shown to be effective in representing nodes and edges, \textit{i.e.} $k=1$ and $k=2$ representations, moving to $k>2$ joint representations requires a model with higher-order factorization. To this end, we introduce \method, a model that leverages hypergraph Markov networks and GNNs to generate $k$-node motif representations.

{\bf Joint $k$-node representations with dyadic graph.} Recently, \citet{morris2019weisfeiler} and \citet{maronBSL19} proposed higher-order neural architectures to represent entire graphs in a supervised learning setting, as opposed to the unsupervised setting discussed in this work. Moreover, we also point how since these higher-order GNN approaches are concerned with representing entire graphs in a supervised setting, the subgraph size $k$ is treated as a constant and scalability is not addressed (models already overfit with small $k$). Our approach can incorporate higher-order GNNs and also the more recent Relational Pooling framework~\citep{murphy2019relational}(see \Cref{eq:energy}). We can summarize previous efforts to represent subgraphs in an unsupervised manner as sums of the individual nodes' representations~\citep{HamiltonSurvey}. Hypergraph neural network models~\citep{yadati2019hypergcn,bai2019hypergraph,feng2019hypergraph} require observing polyadic data, while here we are interested in modeling dyadic data. We provide a broader discussion of higher-order graph models and the challenges of translating supervised approaches to an unsupervised setting in the \Appendix.

{\bf Energy-based models.} Energy-Based Models (EBMs) have been widely used to learn representations of images~\citep{ranzato2007efficient}, text~\citep{bakhtin2020energy}, speech~\citep{teh2003energy} and many other domains. In general, works in EBMs come in two flavors: what model to use for the energy and how to estimate the partition function $Z(\bW)$, which is usually intractable. For the latter, there are model-specific MCMC methods, such as Contrastive Divergence~\citep{hinton2002training} and standard solutions, such as the one we choose in this work: Noise-Constrastive Estimation (NCE). As for the energy model, we opt for a hypergraph Markov network~\citep{rowland2017uprooting,zheleva2010higher,kohli2009robust}. 
The energy of a graph is given by all of its $\binom{n}{k}$ subgraphs, which induces a new kind of intractability in the energy computation. Thus, we propose an unbiased energy estimation procedure in \Cref{sec:opt}, which provides an upper bound on our NCE objective.

{\bf Unrelated work.} 
It is important to not confuse \emph{learning inductive unsupervised \highorder} representations and other existing graph representation methods~\citep{morris2019weisfeiler,maronBSL19,grover2016node2vec,perozzi2014deepwalk,rossi2020structural,rossi2018deep,lee2019graph,yanardag2015deep}. Although motif-aware methods~\cite{lee2019graph,rossi2018deep} explicitly use motif information, they are used to build node representations rather than joint $k$-node representations, and thus, are equatable to other more powerful node representations, such as those in \citet{Hamilton2017,velickovic2018graph}. Here, we are interested in inductive tasks, hence transductive node representations, like \citet{grover2016node2vec,perozzi2014deepwalk}, are unrelated. 
Nevertheless, as a matter of curiosity, we provide results for transductive node representations in our \highorder tasks in the \Appendix, showing that our approach also works well compared to transductive settings even though our approach was not designed for transductive tasks. Supervised higher-order approaches~\citep{morris2019weisfeiler,maronBSL19} extract whole-graph representations, which cannot be directly translated to existing unsupervised settings (see \Appendix for more details on these challenges). We are interested in methods that can be used in end-to-end representation learning, thus feature engineering and extraction, such as those used in graph kernels~\citep{yanardag2015deep} are not of interest. Finally, the large body of work exploring hyperlink prediction in hypergraphs~\citep{benson2018simplicial,patil2020negative,zhang2016recovering,xu2013hyperlink,yoon2020much,zhang2018hyperlinks} requires observing polyadic data (hypergraphs) and are transductive, as opposed to our work, where we consider observing dyadic data and propose an inductive model.

\section{Motif Hypergraph Markov Graph Neural Networks (\method)}\label{sec:model}

In this section we start by introducing notation to then briefly introduce hypergraph Markov networks (HMNs), describe \method with an HMN model, and discuss possible GNN-based energy functions used to represent motifs.

{\bf Notation.} The $i$-th row of a matrix $\mM$ will be denoted $\mM_{i\cdot}$, and its $j$-th column $\mM_{\cdot j}$. For the sake of simplicity, we will focus on graphs without edge attributes, even though our model can handle them using the GNN formulation from \citet{battaglia2018relational}. We denote a graph with $n$ nodes by $G=(V,E,\mX)$, where $V$ is the set of nodes, $E \subseteq V^2$ the edge set, $\mA \in \{0,1\}^{n \times n}$ its corresponding adjacency matrix and the matrix $\mX \in \mathbb{R}^{n \times p}$ encodes the $p$ node features of all $n$ nodes. Each set of $k$ nodes from a graph $C \subseteq V : |C| = k$ has an associated induced subgraph  $G^{(C)} = (V^{(C)}, E^{(C)}, \mX^{(C)})$ (see \Cref{def:subgraphs}). Induced subgraphs are also referred to as \emph{motifs, graphlets, graph fragments or subgraphs}. Here, we will interchangeably refer to them as (induced) subgraphs or motifs.

\begin{definition}[Induced Subgraph]\label{def:subgraphs}
	Let $C \subseteq V : |C| = k$ be a set of $k$ nodes from $V$ with corresponding sorted sequence $\overrightarrow{C} = [C_1, ..., C_k] : C_i < C_{i+1}, C_i \in C \text{ } \forall i \in \{1,...,k\}  $. Then, $G^{(C)} = (V^{(C)}, E^{(C)}, \mX^{(C)})$ is the induced subgraph of $C$ in $G$, with adjacency matrix $\bA^{(C)}$, where $V^{(C)} = \{1, ..., k\}$,
	$\mA \in \{0,1\}^{k \times k} : \mA^{(C)}_{ij} = \mA_{C_i C_j}$ and $\mX^{(C)} \in \mathbb{R}^{k \times p} : \mX^{(C)}_{i\cdot} = \mX_{C_i \cdot}$.
\end{definition}

\textbf{Hypergraph Markov Networks (HMNs).} A Markov Network (MN) defines a joint probability distribution as a product of non-negative functions (potentials) over maximal cliques of an undirected graphical model~\citep{kindermann1982markov,barber2012bayesian}. Although defined over maximal
cliques, scalable techniques often assume factorization over non-maximal cliques~\citep{rowland2017uprooting,barber2012bayesian, zheleva2010higher}, such as Pairwise Markov Networks (PMNs)~\citep{hofling2009estimation}, where the distribution is expressed as a product of edge potentials. In contrast, since we are interested in  learning joint representations of $k$-node subgraphs, we need a hypergraph Markov network (HMN) (\Cref{def:hypnets}), which is an MN model that can encompass all the variables of a $(k>2)$-node subgraph.

Our graph model is an HMN.
HMNs are to PMNs what hypergraphs are to graphs. In HMNs, the joint distribution is expressed as a product of potentials of hyperedges rather than edges. Since in HMNs potentials are defined over subsets of random variables of any size, we have the flexibility to do it over $k$-node subgraphs. There are previous works referring to HMNs as higher-order graphical models~\citep{rowland2017uprooting,zheleva2010higher}, however we find the hypergraph analogy more clarifying.
Next, we provide a formal definition of HMNs.

\begin{definition}[Hypergraph Markov Networks (HMNs)]\label{def:hypnets}
A hypergraph Markov network is a Markov network where the joint probability distribution of $ \bY = \{Y_1, ..., Y_l\}$ can be expressed as $\pr(\bY = \by) = \frac{1}{Z}  \Pi_{h \in \cH} \phi_h(\by_h)$, where $Z$ is the partition function $Z=\sum_{\by' \in \bY} \Pi_{h \in \cH} \phi_h(\by'_h)$, $\phi_h(\cdot) \geq 0$ are non-negative, $\cH \subseteq \cP(\bY) \backslash \{\emptyset\} $, where $\cP(\bY)$ is the powerset of a set $\bY$,
and $\cH$ is the set of hyperedges in the Markov network, $\bY_h$ are the random variables associated with hyperedge $h$ and $\by,\by_h$ assignments of $\bY$ and $\bY_h$ respectively. 
Finally, an energy-based HMN assumes strictly positve potentials, resulting in the model $\pr(\bY = \by) = \frac{1}{Z}  \Pi_{h \in \cH} \exp(-\phi_h(\by_h)) = \frac{1}{Z}  \exp( - \sum_{h \in \cH} \phi_h(\by_h))$, where $\phi_h(\cdot)$ is called the energy function of $h$.
\end{definition}

\subsection{\method{}s}\label{sec:method}

\vspace{-0.02in}

We model $\pr(\bA, \mX | \bW)$ with an energy-based HMN, as described in Definition \ref{def:hypnets}, where a hyperedge corresponds to an induced subgraph of $k$ nodes in the graph $G$. More precisely, for every set of $k > 1$ nodes in the graph $C \subseteq V, |C| = k $, we define a hyperedge $h = \{\bA_{ij} : (i,j) \in C^2\} \cup \{\mX_{i,\cdot} : i \in C\}$ in the HMN to encompass every node variable in the $k$-node set and every edge variable with both endpoints in it. A hyperedge can be indexed by a set of nodes $C$, since its corresponding set of random variables is given by the features $\mX^{(C)}$ and the adjacency matrix $\bA^{(C)}$ of the subgraph induced by $C$, following Definition \ref{def:subgraphs}. Thus, a graph with $n$ nodes will have an HMN with $\binom{n}{k}$ potentials. We formally define the model in Definition \ref{def:model}.

\begin{definition}[\method]\label{def:model}
	
	Let $\cC^{(k)}$ denote the set of all $\binom{n}{k}$ combinations of $k$ nodes from $G$. We define a hypergraph Markov Network with a set of hyperedges $\{\{\mA_{ij} : (i,j) \in C \} \cup \{ \mX_{i,\cdot} : i \in C \} \text{ } : \text{ } C \in \cC^{(k)} \}$, which following Definitions \ref{def:subgraphs} and \ref{def:hypnets}, entails the model
	\vspace{-0.02in}
	\begin{equation}\label{eq:model}
	\!\prob(\bA ,\! \mX | \bW)\! =\! \frac{\exp \! \big(\! -\! \sum_{ C \in \cC^{(k)}} \phi( \bA^{(C)},\! \mX^{(C)};\! \bW ) \big)}{Z(\bW)},
	\vspace{-0.02in}
	\end{equation}
	where $\phi(\cdot,\cdot;\bW)$ is an energy function with parameters $\bW$ and $Z(\bW)$ is the \textit{partition function} given by $Z(\bW) = \sum_{n=1}^{\infty} \sum_{ \bA' \in \{0,1\}^{ n \times n}} \int_{\mX' \in \mathbb{R}^{n \times p}} $ $\exp( - \sum_{ C \in \cC^{(k)}} \phi( \bA'^{(C)}, \mX'^{(C)}; \bW)) d\mX'$.
	
\end{definition}

Although \method factorizes the total energy of a graph, the model does not assume any conditional independence between edge variables for $k>3$. For $k=2$, the model recovers existing edge-based models and for $k=3$ edge variables are dependent only on edges that share one of their endpoints, recovering the Markov random graphs class~\citep{frank1986markov}. Furthermore, \method will learn a jointly exchangeable distribution~\citep{orbanz2014bayesian} if the subgraph energy function $\phi(.,.;\bW)$ is jointly exchangeable, such as a GNN. 
In the \Appendix we connect \method  assumptions, exchangeability and Exponential Random Graph Models (ERGMs).

{\bf Subgraph energy function and representations.} As mentioned, to have a jointly exchangeable model with \method, we need an energy function $\phi( \bA^{(C)}, \mX^{(C)}; \bW)$ that is jointly exchangeable with respect to the subgraph $G^{(C)}$. To this end, we break down $\phi( \bA^{(C)}, \mX^{(C)}; \bW)$ into a composition of two functions. First, we compute a jointly exchangeable representation of $G^{(C)}$, then we use it as input to a more general function that assigns an energy value to the subgraph. Following recent GNN advances~\citep{duvenaud2015convolutional,xu2018how,morris2019weisfeiler}, we define the subgraph representation with a permutation invariant ($\text{READOUT}$) function over the nodes' representations given by a GNN, denoted by $h^{(C)}(\bA^{(C)}, \mX^{(C)}; \bW_{\text{GNN}}, \bW_{\text{R}}) = \text{READOUT}(\text{GNN}(\bA^{(C)},\mX^{(C)};\bW_{\text{GNN}}) ;\bW_{\text{R}})$. Usually, the $\text{READOUT}$ function is a row-wise sum followed by a multi-layer perceptron. Note that, although we choose a 1-GNN approach to represent the subgraph here, any jointly exchangeable graph representation can be used to represent the subgraph, such as $k$-GNNs~\citep{morris2019weisfeiler} and Relational Pooling~\citep{murphy2019relational}.

Finally, we can define the energy of a subgraph $G^{(C)}$ as
\begin{equation}
\label{eq:energy}
\begin{split}
\phi(\bA^{(C)}, \mX^{(C)}; \bW) = \bW_{\text{energy}}^T 
 \rho(h^{(C)}(\bA^{(C)}, \mX^{(C)}; \bW_{\text{GNN}}, \bW_{\text{R}});\bW_{\rho})
\end{split}
\end{equation}
where the model set of weights is $\bW = \{ \bW_{\text{energy}}, \bW_{\text{R}}, \bW_{\rho}, \bW_{\text{GNN}}\}$, $\rho(\cdot;\bW_{\rho})$ is a permutation sensitive function with parameters $\bW_{\rho}$ such as a multi-layer perceptron with range in $\mathbb{R}^{1 \times H}$ and $\bW_{\text{energy}} \in \mathbb{R}^{1 \times H}$ is a (learnable) weight matrix.

Although the functional form of the distribution and subgraph representations are properly defined, directly computing both the partition function and the total energy of a graph are computationally intractable for an arbitrary $k$. Therefore, in the next section we discuss how to properly learn the distribution parameters, providing a principled and scalable approximate method.

\vspace{-0.08in}

\section{Learning \method{}s} \label{sec:opt}

In this section, we first define our unsupervised objective through Noise-Contrastive Estimation (NCE) and then show how to approximate it.

{\bf Noise-Contrastive Estimation (NCE).} Since directly computing $Z(\bW)$ of \Cref{eq:model} is intractable, we use
Noise-Contrastive Estimation (NCE)~\cite{gutmann2010noise}. 
In NCE, the model parameters are learned by contrasting observed data and negative (noise) sampled examples. 
 Given the set $\cD_{\text{true}}$ of observed graphs and $M|\cD_{\text{true}}|$ sampled noise graphs from a noise distribution $\pr_n(\bA , \mX)$ composing the set $D_{\text{noise}}$, we can define the loss function to be minimized as
\vspace{-0.05in}
\begin{align*} 
 \cL(\bA,\mX;\bW) =
- {\sum_{\bA \in \cD_{\text{true}}}} {\log} ( \hat{y}(\Phi(\bA,\mX;\bW),\pr_n(\bA,\mX)) ) \\ - {\sum_{\bA \in \cD_{\text{noise}}}} \log ( 1 - \hat{y}(\Phi(\bA,\mX;\bW),\pr_n(\bA,\mX)) ). 
\label{eq:loss}
\end{align*}

with $\hat{y}(\Phi(\bA,\mX;\bW),\pr_n(\bA,\mX)) =  \sigma(-\Phi(\bA,\mX;\bW) - \log(M\pr_n(\bA,\mX)))$, where $\sigma(\cdot)$ is the sigmoid function and 
$
\Phi(\bA,\mX;\bW) = \sum_{ C \in \cC^{(k)}} \phi( \bA^{(C)}, \mX^{(C)}; \bW )$
denotes the total energy of a graph $G=(V,E,\mX)$ in \method.

If the largest graph in $\cD_{\text{true}} \cup \cD_{\text{noise}}$ has $n$ nodes, directly computing the gradient of the loss $\nabla \cL(\bA,\mX;\bW)$ would take $\mathcal{O}(M|\cD_{\text{true}}|^2 n^{k})$ operations. Traditional Stochastic Gradient Descent (SGD) methods get rid of the dataset size $M|\cD_{\text{true}}|^2$ term by uniformly sampling graph examples. Thus, naively optimizing the NCE loss with SGD would still require $\mathcal{O}(n^{k})$ operations to compute $\Phi(\bA,\mX;\bW)$. 
In what follows we rely on a stochastic optimization procedure that requires a
finite-sample  unbiased estimator of $\Phi(\bA,\mX;\bW)$, where we can also control the estimator's variance with a hyperparameter. 
We show that the resulting stochastic optimization is theoretically sound by proving that it optimizes an upper bound of the original loss.

{\bf Estimating the \method energy $\Phi(\textbf{A},
	\texorpdfstring{\bm{X}} \text{;} \textbf{W})$.} To estimate $\Phi(\bA,\mX;\bW)$, we need to first observe that ---due to sparsity in real-world graphs--- an arbitrary set of $k$ nodes from a graph will induce an empty subgraph with high probability~\citep{newman2018networks}. Therefore, to estimate $\Phi(\bA,\mX;\bW)$ with low variance, we focus on estimating it on \textit{connected induced subgraphs} (CISes)~\citep{teixeira2018graph}, while assuming some constant high energy for disconnected subgraphs. To this end, if $\cC^{(k)}_{\text{conn}}$ is  the set of all $k$-node sets that induce a connected subgraph in $G$, we are now making the reasonable assumption 
\begin{equation}\label{eq:PhiConn}
\Phi(\bA,\mX;\bW) = \sum_{C \in \cC^{(k)}_{\text{conn}}} \phi( \bA^{(C)}, \mX^{(C)}; \bW ) + \text{constant},
\end{equation}
where w.l.o.g.\ we assume the constant to be zero.
Since enumerating all CISes is computationally intractable for arbitrary $k$~\citep{bressan2017counting}, we introduce next a finite-sample unbiased estimator for $\Phi(\bA,\mX;\bW)$ of \Cref{eq:PhiConn} over CISes, denoted by $\widehat{\Phi}(\bA,\mX;\bW)$.

We start by presenting the concept of the higher-order network ($k$-HON) of a graph $G$ 
and its variant called {\em collapsed node} HON ($k$-CNHON).  
An ordinary $k$-HON $G^{(k)}$
is a network where the nodes $V^{(k)}$ 
correspond to $k$-node CISes from $G$ and edges $E^{(k)}$ connect two CISes that share $k-1$ nodes. On the other hand, a $k$-CNHON or $G^{(k,\cI)}$ is a multigraph where a subset of the nodes of  $G^{(k)}$, $\cI \subset V^{(k)}$,
are collapsed into a single node in $G^{(k,\cI)}$.
The collapsed node, henceforth denoted {\em the supernode}, is now node $v^{(k)}_\cI$ in $G^{(k,\cI)}$.
The edges in $G^{(k)}$ of the collapsed nodes $v \in \cI$ among themselves, i.e., the edges in $\cI \times \cI$, do not exist in $G^{(k,\cI)}$.
The edges between the collapsed nodes $v \in \cI$ and other nodes $V \setminus \cI$ are added to $G^{(k,\cI)}$ by replacing the endpoint $v$ with endpoint $v^{(k)}_\cI$, making $G^{(k,\cI)}$ a multigraph (a graph with multiple edges between the same two nodes).
All the remaining edges in $G^{(k)}$ are preserved in $G^{(k,\cI)}$. In \Cref{fig:rwt} we show a graph and its $k$-CNHON with a Random Walk Tour (\Cref{def:rwt}) example.
A formal definition is given in \Appendix.

\begin{definition}[Random Walk Tour (RWT)]\label{def:rwt}
	Consider a simple random walk over a multigraph starting at node $v_{\text{init}}$. A Random Walk Tour (RWT) 
	is represented by a sequence of nodes 
	$\cT = \{v_1, ..., v_{t},v_{t+1}\}$ visited by the random walk such that $v_1 = v_{\text{init}}$, $v_{t+1}=v_{\text{init}}$ and $v_i \neq v_{\text{init}} \text{ } \forall \text{ } 1<i<t+1$.
\end{definition}
In this work, we construct the estimator $\widehat{\Phi}(\bA,\mX;\bW)$ via \textit{random walk tours} (RWTs) on the $k$-CNHON $G^{(k,\cI)}$ starting at the collapsed node $v^{(k)}_\cI$ (i.e,  $v_{\text{init}}=v^{(k)}_\cI$ in \Cref{def:rwt}). 
As previously introduced and discussed in \citet{Avrachenkov2016} and \citet{teixeira2018graph}, increasing the number of tours and the supernode size allow for variance reduction. 
Using these insights, we propose the estimator $\widehat{\Phi}(\bA,\mX;\bW)$, whose properties are defined in \Cref{thm:estimator}.

\begin{theorem}\label{thm:estimator}
	Let $G^{(k)}$ be the $k$-HON of a  graph $G$, a set $\cI$ of $k$-node sets that induce CISes in $G$ (as described above)
	and  $N^{(k)}(C)$ the set of neighbors of the corresponding node of CIS $C$ in $G^{(k)}$. %
	In addition, consider the sample-path $\cT^r = ( v^r_{1}, ..., v^r_{{t}^r}, v^r_{t^r+1})$ visited by the $r$-th RWT on $G^{(k,\cI)}$ starting from supernode 
	$v^{(k)}_\cI$, 
	where $v^r_{i}$ is the node reached at step $i$  for $1 \leq r \leq q$ (\Cref{def:rwt}), and $q \geq 1$ is the number of RWTs. 
	Since $\cT^r$ is a RWT, $v^r_1 = v^{(k)}_\cI$, $v^r_{t^r+1} = v^{(k)}_\cI$ and $v^r_i \neq v^{(k)}_\cI : 1 < i < t^r+1$.
	The nodes $( v^r_{2}, ..., v^r_{{t}^r})$ in the sample path $\cT^r$ have a corresponding sequence of induced $k$-node subgraphs in the graph $G$, denoted $\cT_C^r = (C^r_{i})_{i=2}^{t^r}$. 
	Then, the estimator
	\begin{equation}\label{eq:estimator}
	 \widehat{\Phi}(\bA,\mX;\bW) {=}  {\underbrace{{\sum_{v \in \cI}} \phi{( \bA^{(v)}, \mX^{(v)}; \bW)}}_{\text{Energy of $k$-node CISes in } \cI \text{ (supernode)}}}
	{+} \underbrace{\Big(\frac{\sum_{u \in \cI}|N^{(k)}(u) \backslash \cI|}{q}\Big)  \sum_{r=1}^{q} \sum_{i=2}^{t^r} \frac{\phi( \bA^{(C^r_i)}, \mX^{(C^r_i)}; \bW )}{|N^{(k)}(C^r_i)|} }_{\text{RWT-estimated energy of remaining $k$-node CISes in } G} \nonumber
	\end{equation}
	is an unbiased and consistent estimator of  $\Phi(\bA,\mX;\bW)$ in \Cref{eq:PhiConn} with constant=0.
\end{theorem}

The proof of \Cref{thm:estimator} is in the \Appendix.

We can now replace $\Phi(\bA,\mX;\bW)$ in $\cL(\bA,\mX;\bW)$ with its estimator $\widehat{\Phi}(\bA,\mX;\bW)$, resulting in a loss estimate $\widehat{\cL}(\bA,\mX;\bW)$. It follows from \Cref{thm:estimator} and Jensen's inequality that our loss estimate is in expectation an upper bound to the true NCE loss, \textit{i.e.} $\E_{\widehat{\Phi}}[\widehat{\cL}(\bA,\mX;\bW)] \geq \cL(\bA,\mX;\bW)$. Moreover, note that using an estimator of this nature in higher-order GNNs, such as $k$-GNNs~\citep{morris2019weisfeiler}, does not allow for a bound in the loss estimation (please, see  the \Appendix for further discussion).
Note that the variance of $\widehat{\Phi}$ is controlled by the hyperparameter $q$, the number of tours.
\vspace{-0.5em}

\begin{figure}
\centering
\begin{subfigure}{.4\textwidth}
  \centering
  \includegraphics[width=.4\linewidth]{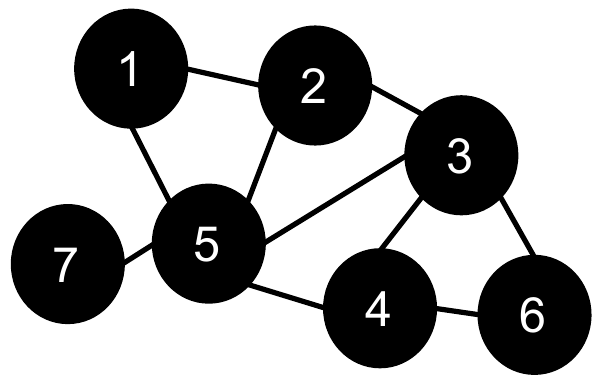}
  \caption{Original graph.}
  \label{fig:sub1}
\end{subfigure}%
\begin{subfigure}{.7\textwidth}
  \centering
  \includegraphics[width=.7\linewidth]{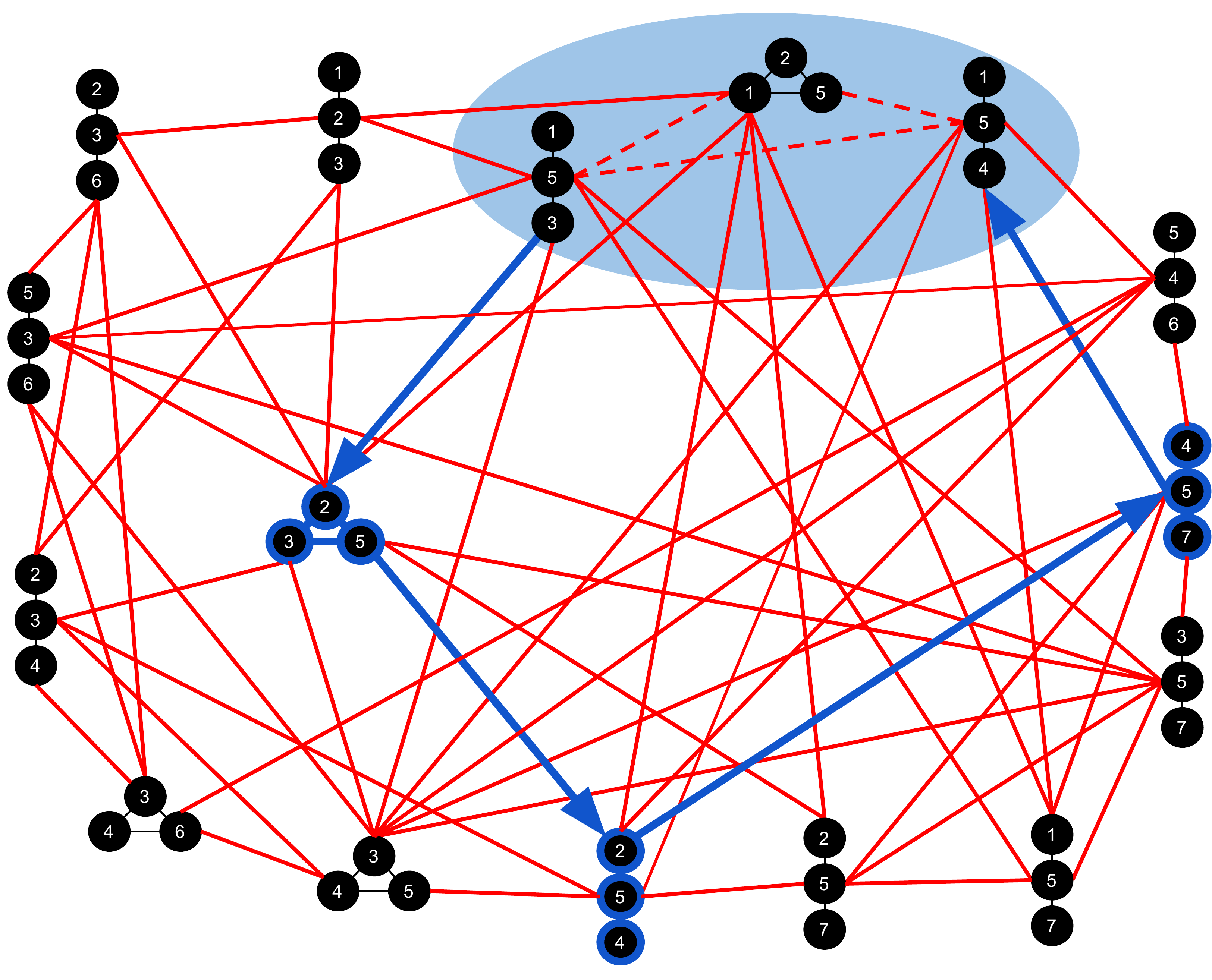}
  \caption{$k$-CNHON with a supernode of size 3 highlighted in blue and an RWT example. Dashed red edges exist in the $k$-HON but are removed in the $k$-CNHON.}
  \label{fig:sub2}
\end{subfigure}
\caption{A graph and its corresponding $k$-CNHON with an RWT example.}
\vspace{-1em}
\label{fig:rwt}
\end{figure} %
\section{Results} \label{sec:exp}
\vspace{-0.5em}
In this section, we evaluate the quality of the unsupervised motif representations learned by \method over six datasets using two \highorder transfer learning tasks. The tasks consider three citation networks, one coauthorship network and two product networks to show how the pre-trained motif representations consistently outperform pooling pre-trained node representations in predicting $k$-node hidden hyperedge labels in downstream tasks --- details of these tasks are in the {\em Hyperedge Detection} and {\em DAG Leaf Counting} subsections.

A good $k$-node representation of a graph is able to capture hidden $k$-order relationships while only observing pairwise interactions. To this end, our tasks evaluate the quality of the unsupervised representations using two hidden hyperedge label prediction tasks. Using the pre-trained unsupervised learned representation as input, we train a simple logistic regression classifier to predict the hidden hyperedge label of a $k$-node set.

{\bf Datasets.} We use the Cora, Citeseer and Pubmed~\citep{sen2008collective} citation networks, the DBLP coauthorship network~\citep{yadati2019hypergcn}, the Steam~\citep{pathak2017generating} and the Rent the Runway~\citep{misra2018decomposing} product networks (more details about the datasets are in the \Appendix).
These datasets were chosen since they contain \highorder information. In the coauthorship network, nodes correspond to authors and edges to the coauthorship of a paper, hidden from the training data we also have the papers and their corresponding author list. In the product networks, nodes correspond to products and an edge exists if the same user bought the two end-point products, hidden from the training data we have the list of products each user bought. In the citation network, nodes correspond to papers and edges to citations, hidden from the training data we also have the direction in which the citation occurred. These directions, paper author list and users purchase history which are hidden in the training data used by the unsupervised GNN and \method representations, give us two transfer learning {\em $k$-node downstream tasks}, described in what follows.

{\bf Hyperedge Detection.} This hyperedge task, inspired by~\citet{yadati2019hypergcn}, creates a $k$-node hyperedge in a citation network whenever a paper cites $k-1$ other papers, in a coauthorship network whenever $k$ authors write a paper together and in a product network whenever a user buys $k$ products. Examples are in the citation networks $k$-size subgraphs with at least one node with degree $k-1$ and $k$-cliques in the other networks.
    Note how \citet{yadati2019hypergcn} directly learns its representations from the hypergraph, a significantly easier task.
    The downstream classifier ---a simple logistic regression classifier--- uses the unsupervised pre-trained representations to classify whether a set of $k$ nodes forms a (hidden) hyperege or not.
    This task allows us to compare the quality of the unsupervised node representations of GNNs against that of \method.
    
{\bf DAG Leaf Counting.} 
   This task considers the  citation networks. Again, baselines and \method are trained over the undirected graphs. %
   Due to the temporal order of citations, subgraphs correspond to Directed Acyclic Graphs (DAGs) in the directed structure. For a connected $k$-node induced subgraph in the directed graph, we want to predict the number of leaves of the resulting DAG. %
    Again, the downstream classifier ---a simple logistic regression classifier--- uses the unsupervised pre-trained representations of a set of $k$-nodes to predict the exact number of leaves formed by the (hidden) $k$-node DAG.
    The number of leaves defines the number of influential papers in the $k$-node set.

{\bf \method architecture.} The energy function of \method is as described in \Cref{eq:energy}, where we use a one-hidden layer feedforward network with LeakyReLU activations as $\rho$, a row-wise sum followed by also a one-hidden layer feedforward network with LeakyReLU activations as the $\text{READOUT}$ function and a single layer GraphSAGE-mean~\citet{Hamilton2017} as the GNN.

{\bf Training the model.} Since the datasets used in this section contain only one large graph for training ---as in most of the real-world graph datasets--- we need to construct a larger set of positive examples $\cD_{\text{true}}$ to learn the distribution $\pr(\bA , \mX | \bW)$. One way to overcome this issue is by subsampling the original large graph. While sampling smaller graphs that preserve the original graph properties, we can approximate the true $\pr(\bA , \mX | \bW)$ distribution and control the complexity of $\widehat{\Phi}(\bA,X;\bW))$ (since tour return times are affected by the size of the graph). To this end, we construct $\cD_{\text{true}}$ by subsampling the original graph with Forest Fire~\citep{leskovec2006sampling}. As for the noise distribution, we turn to the one used by \citet{velickovic2018graph}, where for each positive example we generate $M$ negative samples by keeping the adjacency matrix and shuffling the feature matrix. This noise distribution allows us to keep structural properties of the graph, \textit{e.g.} connectivity, while significantly changing how node features affect the distribution. We precisely describe all hyperparameters and hyperparameter tuning in the \Appendix.

{\bf Experimental setup.} To evaluate the performance of the pre-trained \method representations in the above downstream tasks, we first train the model accordingly for $k=3$ and $k=4$ motif sizes over all six datasets. In the citation and coauthorship networks, we have a single graph, thus these tasks require dividing the graph into training and test sets when evaluating the representations, such that the distribution of observed subgraphs is preserved. To this end, for each dataset, we perform min-cut clustering and use the two cuts for training and test data in the downstream task. For the product networks, to explore the inductive nature of our method, we create two graphs, one for training the models and one for testing the representations. For the Steam dataset, we train on the user-product data from 2014 and test considering the data from 2015. Similarly, for the Rent the Runway dataset, we train on data from 2016 and test on data from 2017. \Cref{tab:new-hyp3,tab:new-hyp4} show our results for the Hyperedge Detection and \Cref{tab:dag} for the DAG Leaf Counting tasks.
\method uses motif sizes $k=3,4$. In the \Appendix, we also show results for $k=5$.
For each task (and $k$), we report the mean and the standard deviation of the balanced accuracy (mean recall of each class) achieved by logistic regression over five different runs. Furthermore, the pre-trained representations (baselines and our approach) have dimension 128. Additional implementation details and hyperparameter search can be found in the \Appendix.

\vspace{-0.em}
{\bf Baselines.} We evaluate the motif representations from \method against two alternatives representing the $k$ nodes using state-of-the-art unsupervised GNN representations: GraphSAGE~\citep{Hamilton2017} and Deep Graph Infomax (DGI)~\citep{velickovic2018graph}.
As a naive baseline, we compare against summing the original features from the nodes, \textit{i.e.} a representation that ignores structural information. 
Moreover, we compare our pre-trained \method representations with an untrained (random parameters) version of \method. Further, since the citation and coauthorship networks consider single graphs, in the \Appendix we show results for these datasets with two prominent transductive node embedding methods~\citep{perozzi2014deepwalk,grover2016node2vec}, evidencing how even in transductive settings node embeddings fail to capture \highorder relationships.

{\bf Results.} The hidden hyperedge downstream tasks are designed to better understand how well pre-trained unsupervised representations can capture \highorder properties. A good \highorder representation should be able to disentangle (hidden) polyadic relationships, even though they only have access to dyadic data. In our Hyperedge Detection task using pre-trained unsupervised representations, \Cref{tab:new-hyp3,tab:new-hyp4} show that \method representations consistently outperform GNN node representations across all datasets. In particular, \method increases classification accuracy by up to 11\% over the best-performing baseline. The results of our the DAG Leaf Counting task, shown in \Cref{tab:dag}, reinforce that pre-trained unsupervised  \method  representations can better capture \highorder interactions.
In particular, \method representations observe classification accuracies by up to 24\% in this downstream task.

\textit{Node representations and \highorder graph tasks.} Our experiments further validate the theoretical claims in  \citet{srinivasan2020equivalence}, that structural node representations are not capable of performing \highorder tasks. That is, the inductive node representations baselines perform similarly to a random classifier in most settings in \Cref{tab:new-hyp3,tab:new-hyp4,tab:dag}. In contrast, the greater accuracy of \method shows that \highorder representations are informative.
 
\textit{Ablation study.} 
As an ablation, we test whether our optimization in \method{} improves the unsupervised \highorder representations, when compared against random neural network weights.
And while \Cref{tab:new-hyp3,tab:new-hyp4,tab:dag} show that \method with random weights perform well in the tasks, since they are effectively a type of motif feature,
the higher accuracy of the optimized \highorder representations shows that the optimized representations in \method are indeed learned.

\begin{table*}
	\centering
	\caption{Balanced accuracy for the \textbf{Hyperedge Detection} task over subgraphs of size $k=3$. We report mean and standard deviation over five runs.}
	\scalebox{0.7}{
	\begin{tabular}{*{1}{l}*{6}{c}}
		\textbf{Method} &  \multicolumn{1}{c}{\textbf{Cora}} & \multicolumn{1}{c}{\textbf{Citeseer}} & \multicolumn{1}{c}{\textbf{Pubmed}} & \multicolumn{1}{c}{\textbf{DBLP}} & \multicolumn{1}{c}{\textbf{Steam}} & \multicolumn{1}{c}{\textbf{Rent the Runway}} \\
		
		{}   & $k=3$  & $k=3$ & $k=3$  & $k=3$  & $k=3$  & $k=3$ \\
		
		\toprule
		
	GS-mean$^\text{\citep{Hamilton2017}}$
	& 0.490 $\pm$ 0.03
	& 0.509 $\pm$ 0.07
	& 0.499 $\pm$ 0.00
	& 0.560 $\pm$ 0.08
	& 0.565 $\pm$ 0.01
	& 0.665 $\pm$ 0.00
	\\	
	GS-max$^\text{\citep{Hamilton2017}}$
   & 0.486 $\pm$ 0.04 
   & 0.493 $\pm$ 0.06
   &  0.498 $\pm$ 0.00
   & 0.569 $\pm$ 0.06
   & 0.579 $\pm$ 0.02
   & 0.667 $\pm$ 0.00
\\
	GS-lstm$^\text{\citep{Hamilton2017}}$    
	& 0.483 $\pm$ 0.04
	& 0.486 $\pm$ 0.05
	& 0.510 $\pm$ 0.02
	& 0.585 $\pm$ 0.06
	& 0.518 $\pm$ 0.01
	& 0.518 $\pm$ 0.01
	\\
	DGI$^\text{\citep{velickovic2018graph}}$
	& 0.487 $\pm$ 0.03
	& 0.508 $\pm$ 0.07
	&  0.509 $\pm$ 0.02
	& 0.497 $\pm$ 0.00
	& 0.588 $\pm$ 0.01
	& 0.612 $\pm$ 0.00
	\\
	Raw Features
	& 0.499 $\pm$ 0.00 
	& 0.588 $\pm$ 0.00 
	& 0.502 $\pm$ 0.00
	& 0.518 $\pm$ 0.00
	& 0.534 $\pm$ 0.00
	& 0.649 $\pm$ 0.00
	\\  \cmidrule(lr){1-1}
	\method (Rnd) 
	& 0.498 $\pm$ 0.00
	& 0.520 $\pm$ 0.05
	& 0.498 $\pm$ 0.01
	& 0.491 $\pm$ 0.01
	& 0.571 $\pm$ 0.01
	& 0.650 $\pm$ 0.00
	\\
	\method
	& \textbf{0.618} $\pm$ 0.03
	& \textbf{0.621} $\pm$ 0.01
	& \textbf{0.602} $\pm$ 0.06
	& \textbf{0.773} $\pm$ 0.02
	& \textbf{0.611} $\pm$ 0.01
	& \textbf{0.676} $\pm$ 0.00
	\\
	\end{tabular}
	}
	\vspace{-1em}
\label{tab:new-hyp3}
\end{table*}

\begin{table*}
	\centering
	\caption{Balanced accuracy for the \textbf{Hyperedge Detection} task over subgraphs of size $k=4$. We report mean and standard deviation over five runs.}
	\scalebox{0.7}{
	\begin{tabular}{*{1}{l}*{6}{c}}
		\textbf{Method} &  \multicolumn{1}{c}{\textbf{Cora}} & \multicolumn{1}{c}{\textbf{Citeseer}} & \multicolumn{1}{c}{\textbf{Pubmed}} & \multicolumn{1}{c}{\textbf{DBLP}} & \multicolumn{1}{c}{\textbf{Steam}} & \multicolumn{1}{c}{\textbf{Rent the Runway}} \\
		
		{}   & $k=4$  & $k=4$ & $k=4$  & $k=4$  & $k=4$  & $k=4$ \\
		
	\toprule
		
	GS-mean$^\text{\citep{Hamilton2017}}$
	& 0.450 $\pm$ 0.11
	& 0.544 $\pm$ 0.03
	& 0.524 $\pm$ 0.05
	& 0.530 $\pm$ 0.15
	& 0.640 $\pm$ 0.03
	& 0.851 $\pm$ 0.00
	\\	
	GS-max$^\text{\citep{Hamilton2017}}$
& 0.462 $\pm$ 0.09
& 0.538 $\pm$ 0.04
& 0.558 $\pm$ 0.05
& 0.511 $\pm$ 0.14
& 0.688 $\pm$ 0.01
& 0.855 $\pm$ 0.00
\\
	GS-lstm$^\text{\citep{Hamilton2017}}$ 
	& 0.444 $\pm$ 0.09
	& 0.536 $\pm$ 0.04
	& 0.566 $\pm$ 0.06
	& 0.653 $\pm$ 0.02
	& 0.504 $\pm$ 0.01
	& 0.546 $\pm$ 0.03
	\\
	DGI$^\text{\citep{velickovic2018graph}}$
	& 0.463 $\pm$ 0.10
	& 0.526 $\pm$ 0.04
	&  0.549 $\pm$ 0.06
	& 0.500 $\pm$ 0.00
	& 0.664 $\pm$ 0.02
	& 0.749 $\pm$ 0.03
	\\
	Raw Features
	& 0.529 $\pm$ 0.01 
	& 0.581 $\pm$ 0.00
	& 0.498 $\pm$ 0.02
	& 0.558 $\pm$ 0.01
	& 0.535 $\pm$ 0.01
	& 0.857 $\pm$ 0.00
	\\  \cmidrule(lr){1-1}
	\method (Rnd)
	& 0.490 $\pm$ 0.10
	& 0.478 $\pm$ 0.03
	& 0.510 $\pm$ 0.02
	& 0.492 $\pm$ 0.02
	& 0.679 $\pm$ 0.01
	& 0.832 $\pm$ 0.01
	\\
	\method
	& \textbf{0.575} $\pm$ 0.03
	& \textbf{0.659} $\pm$ 0.08
	& \textbf{0.701} $\pm$ 0.10
	& \textbf{0.740} $\pm$ 0.05
	& \textbf{0.750} $\pm$ 0.00
	& \textbf{0.860} $\pm$ 0.00
	\\
	\end{tabular}
	}
	\vspace{-0.5em}
\label{tab:new-hyp4}
\end{table*}

\begin{table*}[ht]
\vspace{-0.5em}
	\centering
	\caption{Balanced accuracy for the \textbf{DAG Leaf Counting} task over subgraphs of size $k=3$ and $k=4$. We report mean and standard deviation over five runs.}
	\scalebox{0.7}{
		\begin{tabular}{*{1}{l}*{6}{c}}
			\textbf{Method} &  \multicolumn{2}{c}{\textbf{Cora}} & \multicolumn{2}{c}{\textbf{Citeseer}} & \multicolumn{2}{c}{\textbf{Pubmed}}  \\
			
			{}   & $k=3$  & $k=4$  & $k=3$  & $k=4$ & $k=3$  & $k=4$ \\
			\toprule
			
			GS-mean$^\text{\citep{Hamilton2017}}$
			& 0.468 $\pm$ 0.05 
			& 0.245 $\pm$ 0.06
			& 0.492 $\pm$ 0.04  
			& 0.356 $\pm$ 0.02
			& 0.502 $\pm$ 0.00
			& 0.384 $\pm$ 0.03
			\\	
			GS-max$^\text{\citep{Hamilton2017}}$
			& 0.467 $\pm$ 0.06 
			& 0.245 $\pm$ 0.07
			& 0.486 $\pm$ 0.04 
			& 0.347 $\pm$ 0.01
			&  0.499 $\pm$ 0.00  
			& 0.371 $\pm$ 0.03
			\\
			GS-lstm$^\text{\citep{Hamilton2017}}$
			
			& 0.473 $\pm$ 0.05
			& 0.263 $\pm$ 0.07 
			& 0.482 $\pm$ 0.04
			& 0.348 $\pm$ 0.01 
			& 0.507 $\pm$ 0.02
			& 0.372  $\pm$ 0.03   
			\\
			DGI$^\text{\citep{velickovic2018graph}}$
			& 0.478 $\pm$ 0.05 
			& 0.278 $\pm$ 0.07 
			& 0.504 $\pm$ 0.06  
			& 0.350 $\pm$ 0.02 
			& 0.505 $\pm$ 0.02
			&  0.362 $\pm$ 0.03
			\\
			Raw Features
			& 0.501 $\pm$ 0.01  
			& 0.325 $\pm$ 0.00
			& 0.567 $\pm$ 0.00   
			& 0.380 $\pm$ 0.00
			& 0.503 $\pm$ 0.00
			& 0.339 $\pm$ 0.00
			\\  \cmidrule(lr){1-1}
			\method (Rnd)
			& 0.497 $\pm$ 0.00	
			& 0.327 $\pm$ 0.00
			& 0.518 $\pm$ 0.04
			& 0.319 $\pm$ 0.01
			& 0.499 $\pm$ 0.01
			&  0.343 $\pm$ 0.01 
			\\
			\method
			& \textbf{0.593} $\pm$ 0.03	
			& \textbf{0.452} $\pm$ 0.03	
			& \textbf{0.606} $\pm$ 0.01	 
			& \textbf{0.469} $\pm$ 0.02
			& \textbf{0.626} $\pm$ 0.02
			& \textbf{0.475} $\pm$ 0.08
			\\
		\end{tabular}
	}
	\vspace{-1em}
\label{tab:dag}
\end{table*}

As opposed to node GNN representations and other non-compositional unsupervised graph representation approaches, \method does not take graph-wide information as input. Thus, it is natural to wonder to what extent pre-trained \method \highorder representations are informative of the entire graph to which they belong to. Hence, in the \Appendix we consider whole-graph classification as the downstream task. In this setting, we show how composing (by pooling) \method motif representations can perform better than non-compositional methods, further indicating how our learned motif representations can capture the underlying graph distribution $\pr(\bA,\mX;\bW)$.

\vspace{-0.08in} %
\section{Conclusions}
\vspace{-0.08in}
By combining hypergraph Markov networks, an unbiased finite-sample MCMC estimator, and graph representation learning, we introduced \method, a new scalable class of energy-based representation learning methods capable of learning \highorder representations over dyadic graphs in an \emph{inductive unsupervised} manner.
Finally, we show how pre-trained \method representations achieve more accurate results in downstream \highorder tasks. 
The energy-based optimization in this work allows for many extensions, such as designing different $k$-node subgraph representation learning methods, new subgraph function estimators for \method's loss function, and formulating new \highorder tasks. 
\section*{Broader Impact}
This work presents an unsupervised model together with a stochastic optimization procedure to generate $k$-node representations from graphs, such as online social networks, product networks, citation networks, coauthorship networks, etc.
As is the case with any learning algorithm, it is susceptible to produce biased representations if trained with biased data. Moreover, although the representations might be bias free, the downstream task defined by the user might be biased and thus, also produce biased decisions.

\section*{Acknowledgments}

This work was funded in part by the National Science Foundation (NSF) Awards CAREER IIS-1943364, CCF-1918483, and by
the ARO, under the U.S. Army Research Laboratory contract number W911NF-09-2-0053, the Purdue Integrative Data Science Initiative, the Purdue Research Foundation, and the Wabash Heartland Innovation Network.  Any
opinions, findings and conclusions or recommendations expressed in this material are those of the authors and do not necessarily reflect the views of the sponsors. Further, we would like to thank Mayank Kakodkar for his invaluable feedback and discussion on subgraph function estimation.

\FloatBarrier

\small{
\bibliography{refs,Ribeiro-pubs}
\bibliographystyle{apalike}
}

\appendix

\section{The Estimator \texorpdfstring{$\widehat{\Phi}(\textbf{A},
	\texorpdfstring{\bm{X}} \text{;} \textbf{W})$}{}}

\subsection{The \texorpdfstring{$k$-CNHON network}{}}

\begin{definition}[$k$-CNHON of $G$ given $\cI$, or $G^{(k,\cI)}$] \label{def:cnhon}
	Let $G^{(k)} = (V^{(k)},E^{(k)})$ be the higher-order network ($k$-HON) of the input graph $G$, where each node $v^{(k)} \in V^{(k)}$ corresponds to a $k$-node set $C \in \cC^{(k)}_{\text{conn}}$. For ease of understanding, we will levarege this correspondence and refer to nodes from $V^{(k)}$ with $k$-node sets from $\cC^{(k)}_{\text{conn}}$ interchangeably.  The edge set $E^{(k)}$ is defined such that $E^{(k)} = \{ (v^{(k)}_i, v^{(k)}_j) : v^{(k)}_i, v^{(k)}_j \in \cC^{(k)}_{\text{conn}} \text{ and } |v^{(k)}_i \cap  v^{(k)}_j| = k-1 \}$. Moreover, let $\cI$ be a set of k-nodes sets $\cI \subset \cC^{(k)}_{\text{conn}}$. Then, a $k$-CNHON $G^{(k,\cI)} = (V^{(k,\cI)}$, $E^{(k,\cI)}$) with supernode $v^{(k)}_\cI$ is a multigraph with node set $V^{(k,\cI)} = (V^{(k)} \backslash \cI ) \cup v^{(k)}_\cI$ 
	and edge multiset $E^{(k,\cI)} =  E^{(k)} \backslash (E^{(k)}\cap (\cI \times \cI)) \uplus  \{ (v^{(k)}_\cI, v^{(k)}) : \text{ } \exists \text{ } (u^{(k)}, v^{(k)}) \in E^{(k)}, u^{(k)} \in \cI \text{ and } v^{(k)} \notin \cI \}$, where $\uplus$ is the multiset union operation. 
\end{definition}

\subsection{ Proof of \texorpdfstring{\Cref{thm:estimator}}{}}

To prove \Cref{thm:estimator}, we assume that $G^{(k,\cI)}$ has a stationary distribution $\pi$ with 

$$\pi(C_i) = \frac{|N^{(k)}(C_i)|}{\sum_{C' \in V^{(k)}\backslash \cI} |N^{(k)}(C')|
	+ \sum_{u \in \cI} | N^{(k)}(u) \backslash \cI | } \text{ } \forall \text{ } C_i \in V^{(k,\cI)} \backslash \{v_{\cI}^{(k)}\} ,$$

and

$$ \pi(v_{\cI}^{(k)}) = \frac{\sum_{u \in \cI} | N^{(k)}(u) \backslash \cI | }{\sum_{C' \in V^{(k)}\backslash \cI} |N^{(k)}(C')| + \sum_{u \in \cI} | N^{(k)} (u) \backslash \cI |}.$$

Fortunately, \citet{wang2014efficiently} showed that such a statement is true whenever $\cI$ contains at least one $k$-node set from each connected component of $G$ and if each such component contains at least one vertex which is not a part of any $k$-node set in $\cI$ and is contained in more than 2 edges in $G$. 
First, we show that
the estimate $\widehat{\Phi}(\bA,\mX;\bW)$ of each tour is unbiased. 
\begin{lemma}\label{lem:unbias}
	Let $\cT_C^r = (C^r_{i})_{i=2}^{t^r}$ be a $k$-node set chain formed by the samples from the $r$-th RWT on $G^{(k,\cI)}$%
	starting at the supernode $v_{\cI}^{(k)}$. Then, $\forall r \geq 1$,
	\begin{equation}\label{eq:unbias}
	\begin{split}
	\mathbb{E}\Big[ \sum_{v \in \cI} \phi( \bA^{(v)}, \mX^{(v)}; \bW)
	+ \Big(\sum_{u \in \cI}|N^{(k)}(u) \backslash \cI|\Big)  \sum_{i=2}^{t^r} \frac{\phi( \bA^{(C^r_i)}, \mX^{(C^r_i)}; \bW )}{|N^{(k)}(C^r_i)|}\Big] 
	= \Phi(\bA,\mX;\bW),
	\end{split}
	\end{equation}
	
	assuming $\Phi(\bA,\mX;\bW)$ with zero constant.

\end{lemma}

\begin{proof} [Proof of \Cref{lem:unbias}]
	
	Let's first rewrite \Cref{eq:unbias} as
	
	\begin{equation}\label{eq:simple-unbias}
	\begin{split}
	\Big(\sum_{u \in \cI}|N^{(k)}(u) \backslash \cI|\Big)  \mathbb{E}\Big[  \sum_{i=2}^{t^r} \frac{\phi( \bA^{(C^r_i)}, \mX^{(C^r_i)}; \bW )}{|N^{(k)}(C^r_i)|}\Big] 
	= \Phi(\bA,\mX;\bW) - \sum_{v \in \cI} \phi( \bA^{(v)}, \mX^{(v)}; \bW).
	\end{split}
	\end{equation}
	
	Since the RWT starts at node $v_{\cI}^{(k)}$, we may rewrite the expected value in \Cref{eq:simple-unbias} as
	\begin{equation}\label{eq:unbias2}
	\begin{split}
	\mathbb{E}\left[\sum_{i=2}^{t^{r}} \frac{\phi( \bA^{(C^r_i)}, \mX^{(C^r_i)}; \bW )}{|N^{(k)}(C^r_i)|} \right] 
	= \sum_{C_i \in \cC^{(k)}_{\text{conn}} \backslash \cI} \mathbb{E} \left[ \mathbb{\mT}(C_i) \frac{\phi( \bA^{(C_i)}, \mX^{(C_i)}; \bW )}{|N^{(k)}(C_i)|} \right],
	\end{split}
	\end{equation} 
	where $\mathbb{\mT}(C)$ represents the number of times the RWT reaches state $C$.

	Consider a renewal reward process with inter-renewal time distributed as $t^r$, $r \geq 1$ and reward as $\mathbb{\mT}(C^r_i)$. Further, note that the chain is positive recurrent, thus $\mathbb{E}[t^r] < \infty$, $\mathbb{E}[\mathbb{\mT}(C^r_i)] < \infty$ and $\mathbb{\mT}(C^r_i) < \infty$. Then, from the renewal reward theorem and the ergodic theorem~\citep{bremaud2013markov} we have
	
	$$\pi(C^r_i) = \mathbb{E}[t^r]^{-1} \mathbb{E}[\mathbb{\mT}(C^r_i)].$$
	
	Moreover, it follows from Kac's formula~\citep{aldous1995reversible} that $ \mathbb{E}[t^r] = \frac{1}{\pi(v_{\cI}^{(k)})}$.
	Therefore, \Cref{eq:unbias2} can be rewritten as
	\begin{equation}\label{eq:unbias3}
	\begin{split}
	\mathbb{E}\left[\sum_{i=2}^{t^{r}} \frac{\phi( \bA^{(C^r_i)}, \mX^{(C^r_i)}; \bW )}{|N^{(k)}(C^r_i)|} \right]
	= \sum_{C_i \in \cC^{(k)}_{\text{conn}} \backslash \cI} \frac{ \pi(C_i) \phi( \bA^{(C_i)}, \mX^{(C_i)}; \bW )}{\pi(v_{\cI}^{(k)})|N^{(k)}(C_i)|}.
	\end{split}
	\end{equation}
	Now, knowing the stationary distribution of $G^{(k,\cI)}$, we may simplify \Cref{eq:unbias3}
	to
	\begin{equation}\label{eq:unbias4}
	\begin{split}
	\mathbb{E}\left[\sum_{i=2}^{t^{r}} \frac{\phi( \bA^{(C^r_i)}, \mX^{(C^r_i)}; \bW )}{|N^{(k)}(C^r_i)|} \right]  
	=\frac{1}{\sum_{u \in \cI} | N^{(k)}(u) \backslash \cI |}\sum_{C_i \in \cC^{(k)}_{\text{conn}} \backslash \cI} \phi( \bA^{(C_i)}, \mX^{(C_i)}; \bW ),
	\end{split}
	\end{equation}
	and replace it in \Cref{eq:simple-unbias}, concluding our proof.
\end{proof}

\begin{proof}[Proof of \Cref{thm:estimator}]
	By \Cref{lem:unbias}, linearity of expectation and knowing that each RWT is independent from the other tours by the Strong Markov Property, \Cref{thm:estimator} holds.
\end{proof}

\section{Discussion of \method properties}

\paragraph{Conditional independence.} Although HMNs factorize distributions, the potentials themselves do not provide information on conditional and marginal distributions. Rather, we need to analyze how every pair of variables interacts through all potentials. For the sake of simplicity, consider the model described in Definition \ref{def:model} for undirected simple graphs, \textit{i.e.} $\mA_{ij}=\mA_{ji} \text{ } \forall \text{ } (i,j) \in V^2, \bA_{ii} = 0 \text{ } \forall \text{ } i \in V $. If we set $k=2$, each hyperedge will contain exactly one edge variable and two node variables, which is equivalent to assuming all edges are independent given their nodes' representations. Thus, for $k=2$ \method can recover edge-based models where representations don't use graph-wide information. Furthermore, if we allow the node representation to take graph-wide information, we can recover the recent Graph Neural Networks approaches~\cite{Hamilton2017,kipf2016variational,bojchevski2018deep}. If we opt for $k=3$, a hyperedge defined by nodes $i,j,l$ will contain the set of edge variables $\{\bA_{ij},\bA_{il},\bA_{jl}\}$ and node variables $\{ \mX_{i,\cdot},\mX_{j,\cdot},\mX_{l,\cdot}  \}$. Thus, a hyperedge will encompass only edge variables that share one endpoint. In this case, an edge variable $\bA_{ij}$ is independent from $\{\bA_{lm} : l,m \in V , \{l,m\} \cap \{i,j\} = \emptyset \} $ others given $ \{ \bA_{il} : l \in V  \} \cup \{ \bA_{im} : m \in V \} \cup \{ \mX_{i,\cdot} : i \in V \}$ . Thus, \method with $k=3$ can be cast as an instance of the Markov random graphs class proposed by~\citet{frank1986markov}. With $k \geq 4$, for every pair of edge variables $\bA_{ij},\bA_{lm}$ there exists at least one $C \in \cC^{(k)}$ such that $i,j,l,m \subseteq C$. Thus, there exists at least one hyperedge covering every pair of edge variables in the model, resulting in a fully connected hypergraph Markov Network. Therefore, for $k \geq 4$ the model does not assume any conditional independence between edge variables which, since subgraphs share edge variables, is a vital feature for \highorder representations of graphs.
\paragraph{Exchangeability.}
Although with infinite data and an arbitrary energy function $\phi(\cdot,\cdot;\bW)$ \method would learn a jointly exchangeable~\citep{orbanz2014bayesian} distribution, we would like to impose such condition on the model, defining a proper graph model. Equivalently, we would like to guarantee that any two isomorphic graphs have the same probability under \method. By definition, the sets of subgraphs from two isomorphic graphs are equivalent under graph isomorphism. Thus, if the subgraph energy function $\phi(\cdot,\cdot;\bW)$ is jointly exchangeable, the set of subgraph energies from two isomorphic graphs are equivalent. Since the sum operation is permutation invariant and the partition function is a constant, a jointly exchangeable subgraph energy function $\phi(\cdot,\cdot;\bW)$, such as a GNN, is enough to make \method jointly exchangeable.

\paragraph{Exponential Random Graph Models (ERGMs).}

The form of \method presented in \Cref{def:model} resembles the general and classical expression of Exponential Random Graph Models (ERGMs)~\cite{kolaczyk2014statistical}. Indeed, as any energy-based network model, we can cast ours as an ERGM where the sufficient statistics are given by all $k$-size subgraphs. However, we do stress how any exchangeable graph model has a correspondent ERGM representation~\cite{lauritzen2018random}, even when it is not as clear as it in \method.

\section{Additional Experiments and Implementation Details from \texorpdfstring{\Cref{sec:exp}}{}}

\subsection{Results for \texorpdfstring{$k=5$}{}}

Here, we extend the results from \Cref{sec:exp} to a $k=5$ setting in \Cref{tab:dag5} and \cref{tab:hyp5}. Due to the lack of papers with 5 authors (less than 10), we were not able to extend them to the DBLP dataset. Moreover, the conclusions from \Cref{sec:exp} also hold here. That is, \method consistently outperforms the baselines. However, on Rent the Runway we see the raw features achieving the highest performance. That is, structural information does not seem to be relevant to this specific task. Nevertheless, we still see that \method and GraphSAGE are the methods able to perform the task similarly to the raw features.

\begin{table*}[ht]
	\centering
	\caption{Balanced accuracy for the \textbf{Hyperedge detection} task over subgraphs of size $k=5$. We report mean and standard deviation over five runs.}
	\scalebox{0.75}{
		\begin{tabular}{*{1}{l}*{5}{c}}
			\textbf{Method} &  \multicolumn{1}{c}{\textbf{Cora}} & \multicolumn{1}{c}{\textbf{Citeseer}} & \multicolumn{1}{c}{\textbf{Pubmed}}  & \multicolumn{1}{c}{\textbf{Steam}} & \multicolumn{1}{c}{\textbf{Rent the Runway}}   \\
			
			{}   & $k=5$  & $k=5$  & $k=5$ & $k=5$ & $k=5$ \\
			\toprule
			
			GS-mean$^\text{\citep{Hamilton2017}}$
			&  0.447 $\pm$ 0.10
			&  0.530 $\pm$ 0.03
			&  0.697 $\pm$ 0.08
			&  0.696 $\pm$ 0.07
			&  0.933 $\pm$ 0.00
			\\	
			GS-max$^\text{\citep{Hamilton2017}}$
			& 0.384 $\pm$ 0.09
			& 0.543 $\pm$ 0.08
			& 0.722 $\pm$ 0.06
			& 0.765 $\pm$ 0.03
			& 0.940 $\pm$ 0.00
			\\
			GS-lstm$^\text{\citep{Hamilton2017}}$
			
			& 0.422 $\pm$ 0.04
			& 0.525 $\pm$ 0.03 
			& 0.736 $\pm$ 0.08
			& 0.532 $\pm$ 0.05
			& 0.557 $\pm$ 0.05
			\\
			DGI$^\text{\citep{velickovic2018graph}}$
			& 0.504 $\pm$ 0.00
			& 0.500 $\pm$ 0.00 
			& 0.500 $\pm$ 0.00
			& 0.626 $\pm$ 0.11
			& 0.827 $\pm$ 0.04
			\\
			Raw Features
			& 0.500 $\pm$ 0.00
			& 0.513 $\pm$ 0.00
			& 0.526 $\pm$ 0.00
			& 0.602 $\pm$ 0.00
			& \textbf{0.944} $\pm$ 0.00
			\\  \cmidrule(lr){1-1}
			\method (Rnd)
			&  0.460 $\pm$ 0.05
			&  0.453 $\pm$ 0.03
			&  0.493 $\pm$ 0.07
			&  0.748 $\pm$ 0.02
			&  0.924 $\pm$ 0.00
			\\
			\method
			& \textbf{0.543} $\pm$ 0.06
			& \textbf{0.703} $\pm$ 0.04	
			&  \textbf{0.815} $\pm$ 0.10
			& \textbf{0.823} $\pm$ 0.00
			& 0.943 $\pm$ 0.01
			\\
		\end{tabular}
	}
\label{tab:hyp5}
\end{table*}

\begin{table*}[ht]
	\centering
	\caption{Balanced accuracy for the \textbf{DAG Leaf Counting} task over subgraphs of size $k=5$. We report mean and standard deviation over five runs.}
	\scalebox{0.75}{
		\begin{tabular}{*{1}{l}*{3}{c}}
			\textbf{Method} &  \multicolumn{1}{c}{\textbf{Cora}} & \multicolumn{1}{c}{\textbf{Citeseer}} & \multicolumn{1}{c}{\textbf{Pubmed}}  \\
			
			{}   & $k=5$  & $k=5$  & $k=5$ \\
			\toprule
			
			GS-mean$^\text{\citep{Hamilton2017}}$
			& 0.223 $\pm$ 0.04
			& 0.259 $\pm$ 0.02
			& 0.284 $\pm$ 0.02
			\\	
			GS-max$^\text{\citep{Hamilton2017}}$
			& 0.150 $\pm$ 0.07
			& 0.263 $\pm$ 0.01
			&  0.288 $\pm$ 0.02
			\\
			GS-lstm$^\text{\citep{Hamilton2017}}$
			
			& 0.214 $\pm$ 0.03
			& 0.259 $\pm$ 0.00 
			&  0.295 $\pm$ 0.04
			\\
			DGI$^\text{\citep{velickovic2018graph}}$
			& 0.236 $\pm$ 0.02
			& 0.249 $\pm$ 0.00 
			& 0.249 $\pm$ 0.00
			\\
			Raw Features
			& 0.251 $\pm$ 0.05
			& 0.266 $\pm$ 0.00
			&  0.290 $\pm$ 0.00
			\\  \cmidrule(lr){1-1}
			\method (Rnd)
			& 0.231 $\pm$ 0.01
			& 0.277 $\pm$ 0.02
			& 0.244 $\pm$ 0.01
			\\
			\method
			& \textbf{0.363}  $\pm$ 0.04
			& \textbf{0.364} $\pm$ 0.02
			& \textbf{0.330} $\pm$ 0.04
			\\
		\end{tabular}
	}
\label{tab:dag5}
\end{table*}

\subsection{Hyperparameters and Hyperparameter Search for \method}

All \method models were implemented in PyTorch~\citep{paszke2019pytorch} and PyTorch Geometric~\citep{fey2019} with the Adam optimizer~\citep{KingmaB14}. All hyperparameters were chosen to minimize training loss. For learning rate, we searched in \{0.01, 0.001, 0.0001\} finding the best learning rate to be 0.001 for all models. We used a single hidden layer feedforward network with LeakyReLU activations for both $\rho$ and $\text{READOUT}$ functions in all models. Furthermore, following GraphSAGE~\citet{Hamilton2017}, for all models we do an L2 normalization in the motif representation layer, \textit{i.e.} in the output of the $\text{READOUT}$ function. Finally, for all models we use $M=1$ negative example for each positive example. %
In what follows, we give specific hyperparameters and their search for experiments from \Cref{sec:exp}, show results for transductive baselines, and introduce new whole-graph downstream tasks together with their specific hyperparameters and search as well.

\subsection{Pre-trained HMH-GNN for \texorpdfstring{$k$-node downstream tasks (\Cref{sec:exp})}{}} 

{\bf \method architecture.} The energy function of \method is as described in \Cref{eq:energy}, where we use a one-hidden layer feedforward network with LeakyReLU activations as $\rho$, a row-wise sum followed by also a one-hidden layer feedforward network with LeakyReLU activations as the $\text{READOUT}$ function and a single layer GraphSAGE-mean~\citet{Hamilton2017} as the GNN, except for $k=5$ in the citation networks where we used two layers of the GraphSAGE-mean GNN to achieve faster convergence in training.

\textbf{Subsampling positive examples.} We use positive examples subsampled with Forest Fire~\citep{leskovec2006sampling} of size 100 for Cora, Citeseer and DBLP datasets, while for Pubmed, a larger network, we use examples of size 500. For Steam, a smaller network, we use 75 and for Rent the Runway, a mid-size network we use 150.

\textbf{Number of tours.} We did 80 tours for all datasets except Pubmed with $k=4$, which due to a larger $k$-CNHON network, we did 120 tours. A small number of tours will result in high variance in the gradient which, as we observed, tends to impair the learning process. Therefore, we tested training models, each with a different fix number of tours, starting with 1 tour and increasing it 10 by 10 until we reached the reported number of tours, which results in training loss convergence.

\textbf{Supernode size.} To construct the supernode, we do a BFS on the $k$-HON of the original input graph, similarly to ~\citet{teixeira2018graph}. We have a parameter that controls the maximum number of subgraphs visited by the BFS, which we call supernode budget. This parameter was set to 100K for Pubmed with $k=3$ and $k=4$, 5K for Cora with $k=3$ and $k=4$, Citeseer with $k=3$ and DBLP with $k=3$, 10K for Citeseer with $k=4$ and 50K for DBLP with $k=4$. For Steam, we set to 1K for $k=3$ and to 10K for $k=4$. For Rent the Runway, we set to 10K for $k=3$ and to 30K for $k=4$. For $k=5$, we used 50K in Cora, 75K in Citeseer, 120K in Pubmed, 50K in Steam and 100K in Rent the Runway. In the same way of tours, we started with a small supernode budget of 100 and increased it by 100 until we observed the tours being completed and the training loss converging.

\textbf{Minibatch size.} We used a minibatch size of 50 for Cora, Citeseer and Steam with $k=3$ and 25 for Cora and Citesser with $k=4$. For Pubmed, Rent the Runway and DBLP, larger networks, we used minibatches of size 40 for $k=3$ and 10 for $k=4$. For Steam, we used 20 for $k=4$. Again, we tested small minibatch sizes, increasing them until we had training loss convergence and GPU memory space to use. For $k=5$, we used a minibatch of size 5 in all datasets.

\subsubsection{Transductive baselines} 
Since we defined the tasks from \Cref{sec:exp} over single graphs in the citation and couathorship networks, in \Cref{tab:transhyp3,tab:transhyp4,tab:transhyp5}, and \Cref{tab:transdag3,tab:transdag4,tab:transdag5} we show for those datasets results for two prominent transductive node embedding methods, node2vec~\citep{grover2016node2vec} and DeepWalk~\citep{perozzi2014deepwalk} together with concatenating the raw features to them, evidencing how even in transductive settings, transductive node embeddings fail to capture \highorder relationships in most settings, performing similarly to the inductive approaches to node representations, thus, performing consistently worse than our \method \highorder representations.

\begin{table*}
	\caption{Results for transductive baselines in the \textbf{Hyperedge Detection} task over $k=3$, $k=4$ and $k=5$ size subgraphs.}
	\begin{subtable}[ht]{\textwidth}
		\centering
		\scalebox{0.75}{
			\begin{tabular}{lcccc}
				\textbf{Method} & \textbf{Cora} & \textbf{Citeseer} & \textbf{Pubmed} & \textbf{DBLP} \\
				{}  & $k=3$ & $k=3$  & $k=3$  & $k=3$ \\
				\toprule
				node2vec$^\text{\citep{grover2016node2vec}}$               & 0.534 $\pm$ 0.04 & 0.525 $\pm$ 0.02 & 0.501 $\pm$ 0.00 &  0.461 $\pm$ 0.05  \\
				node2vec$^\text{\citep{grover2016node2vec}}$   + Features             & 0.545 $\pm$ 0.01 & 0.534 $\pm$ 0.01 & 0.500 $\pm$ 0.00  & 0.479 $\pm$ 0.04  \\
				DeepWalk$^\text{\citep{perozzi2014deepwalk}}$             & 0.472 $\pm$ 0.02 &  0.433 $\pm$ 0.01 &0.499 $\pm$ 0.00   & 0.481 $\pm$ 0.00  \\
				DeepWalk$^\text{\citep{perozzi2014deepwalk}}$  + Features	& 0.512 $\pm$ 0.01 & 0.591 $\pm$ 0.01 & 0.502 $\pm$ 0.00 & 0.485 $\pm$ 0.02 \\	%
			\end{tabular}
			}
		\caption{$(k=3)$ Balanced accuracy for the \textbf{Hyperedge Detection} task over subgraphs of size $k=3$. We report mean and standard deviation over five runs.}
		\label{tab:transhyp3}
	\end{subtable}
	\begin{subtable}[ht]{\textwidth}
		\centering
		\scalebox{0.75}{
			\begin{tabular}{lcccc}
				\textbf{Method} & \textbf{Cora} & \textbf{Citeseer} & \textbf{Pubmed} & \textbf{DBLP} \\ 
				{}  & $k=4$ & $k=4$  & $k=4$  & $k=4$ \\
				\toprule
				node2vec$^\text{\citep{grover2016node2vec}}$               &  0.537 $\pm$ 0.04 & 0.513 $\pm$ 0.03 & 0.504 $\pm$ 0.01 & 0.405 $\pm$ 0.01  \\
				node2vec$^\text{\citep{grover2016node2vec}}$   + Features             & 0.626 $\pm$ 0.03 &0.540 $\pm$ 0.01 & 0.502 $\pm$ 0.00 & 0.548 $\pm$ 0.10 \\
				DeepWalk$^\text{\citep{perozzi2014deepwalk}}$             & 0.515 $\pm$ 0.07 & 0.494 $\pm$ 0.10  & 0.504 $\pm$ 0.01   &  0.460 $\pm$ 0.01  \\
				DeepWalk$^\text{\citep{perozzi2014deepwalk}}$  + Features	& 0.597 $\pm$ 0.05 & 0.570 $\pm$ 0.01 & 0.516 $\pm$ 0.01 & 0.560 $\pm$ 0.03 \\		%
			\end{tabular}
			}
		\caption{($k=4$) Balanced accuracy for the \textbf{Hyperedge Detection} task over subgraphs of size $k=4$. We report mean and standard deviation over five runs.\\
			~}
		\label{tab:transhyp4}
	\end{subtable}
	\begin{subtable}[ht]{\textwidth}
		\centering
		\scalebox{0.75}{
			\begin{tabular}{lcccc}
				\textbf{Method} & \textbf{Cora} & \textbf{Citeseer} & \textbf{Pubmed}  \\ 
				{}  & $k=5$ & $k=5$  & $k=5$ \\
				\toprule
				node2vec$^\text{\citep{grover2016node2vec}}$               & 0.446  $\pm$ 0.08 & 0.544 $\pm$ 0.08  & 0.623  $\pm$ 0.13   \\
				node2vec$^\text{\citep{grover2016node2vec}}$   + Features             &  0.519 $\pm$ 0.00 & 0.500 $\pm$ 0.00  & 0.502 $\pm$ 0.01  \\
				DeepWalk$^\text{\citep{perozzi2014deepwalk}}$             & 0.446 $\pm$ 0.07 & 0.568 $\pm$ 0.05 & 0.568 $\pm$ 0.13   \\
				DeepWalk$^\text{\citep{perozzi2014deepwalk}}$  + Features	& 0.490 $\pm$ 0.01  & 0.523 $\pm$ 0.01  & 0.472 $\pm$ 0.11 		%
			\end{tabular}
			}
		\caption{($k=5$) Balanced accuracy for the \textbf{Hyperedge Detection} task over subgraphs of size $k=5$. We report mean and standard deviation over five runs.\\
			~}
		\label{tab:transhyp5}
	\end{subtable}
\end{table*}

\begin{table*}[ht]
\caption{Results for transductive baselines in the \textbf{DAG Leaf Counting} task over $k=3$, $k=4$ and $k=5$ size subgraphs.}
	\begin{subtable}[ht]{\textwidth}
		\centering
		\scalebox{0.75}{
			\begin{tabular}{lcccc}
				\textbf{Method} & \textbf{Cora} & \textbf{Citeseer} & \textbf{Pubmed} \\ 
				{}  & $k=3$ & $k=3$  & $k=3$ \\
				\toprule
				node2vec$^\text{\citep{grover2016node2vec}}$               & 0.538 $\pm$ 0.05 & 0.546 $\pm$ 0.03 & 0.502 $\pm$ 0.01 \\
				node2vec$^\text{\citep{grover2016node2vec}}$   + Features             & 0.556 $\pm$ 0.02  & 0.527 $\pm$ 0.01 & 0.501 $\pm$ 0.00 \\
				DeepWalk$^\text{\citep{perozzi2014deepwalk}}$             & 0.466 $\pm$ 0.02 & 0.503 $\pm$ 0.06 & 0.499 $\pm$  0.00   \\
				DeepWalk$^\text{\citep{perozzi2014deepwalk}}$  + Features	&  0.543 $\pm$ 0.01 & 0.584 $\pm$ 0.00 & 0.503 $\pm$ 0.00 \\			%
			\end{tabular}
			}
		\caption{($k=3$) Balanced accuracy for the \textbf{DAG Leaf Counting} task over subgraphs of size $k=3$. We report mean and standard deviation over five runs.}
		\label{tab:transdag3}
	\end{subtable}
	\begin{subtable}[ht]{\textwidth}
		\centering
		\scalebox{0.75}{
			\begin{tabular}{lcccc}
				\textbf{Method} & \textbf{Cora} & \textbf{Citeseer} & \textbf{Pubmed} \\ 
				{}  & $k=4$ & $k=4$  & $k=4$ \\
				\toprule
				node2vec$^\text{\citep{grover2016node2vec}}$               & 0.374 $\pm$ 0.06 & 0.329 $\pm$ 0.04 & 0.333 $\pm$ 0.00 \\
				node2vec$^\text{\citep{grover2016node2vec}}$   + Features             & 0.410 $\pm$ 0.04  & 0.388 $\pm$ 0.00 & 0.339  $\pm$ 0.00  \\
				DeepWalk$^\text{\citep{perozzi2014deepwalk}}$             & 0.322 $\pm$ 0.00 & 0.349 $\pm$ 0.04 &0.339 $\pm$ 0.00    \\
				DeepWalk$^\text{\citep{perozzi2014deepwalk}}$  + Features	& 0.349 $\pm$ 0.00 & 0.381 $\pm$ 0.00  & 0.345 $\pm$ 0.00 \\			%
			\end{tabular}
			}
		\caption{($k=4$) Balanced Accuracy for the \textbf{DAG Leaf Counting} task over subgraphs of size $k=4$. We report mean and standard deviation over five runs.}
		\label{tab:transdag4}
	\end{subtable}
	
	\begin{subtable}[ht]{\textwidth}
		\centering
		\scalebox{0.75}{
			\begin{tabular}{lcccc}
				\textbf{Method} & \textbf{Cora} & \textbf{Citeseer} & \textbf{Pubmed} \\ 
				{}  & $k=5$ & $k=5$  & $k=5$ \\
				\toprule
				node2vec$^\text{\citep{grover2016node2vec}}$               & 0.265  $\pm$ 0.05  & 0.262 $\pm$ 0.03 & 0.298 $\pm$ 0.03  \\
				node2vec$^\text{\citep{grover2016node2vec}}$   + Features             & 0.263 $\pm$ 0.01 & 0.240 $\pm$ 0.02  & 0.259 $\pm$ 0.01   \\
				DeepWalk$^\text{\citep{perozzi2014deepwalk}}$             & 0.254 $\pm$ 0.02 & 0.240 $\pm$ 0.01  & 0.238 $\pm$ 0.05    \\
				DeepWalk$^\text{\citep{perozzi2014deepwalk}}$  + Features	&  0.255 $\pm$ 0.00 & 0.269 $\pm$ 0.00  & 0.269 $\pm$ 0.01 \\			%
			\end{tabular}
			}
		\caption{($k=5$) Balanced Accuracy for the \textbf{DAG Leaf Counting} task over subgraphs of size $k=5$. We report mean and standard deviation over five runs.}
		\label{tab:transdag5}
	\end{subtable}
	
\end{table*}

\subsection{Pre-trained MHM-GNN representations for whole-graph downstream tasks}\label{sec:graphs}
In \Cref{sec:exp}, we have seen that the motif representations learned by \method can better predict hyperedge properties than existing unsupervised GNN representations. 
In the following experiments we investigate: Are \method motif representations capturing graph-wide information (learning $\pr(\bA,\mX;\bW)$)? To this end, inspired by \citet{nairhinton2008}'s evaluation of RBM representations through supervised learning, we now investigate if \method's pre-trained motif representations can do similarly or better than non-compositional methods that take graph-wide information in (inductive) whole-graph classification.%

\textbf{Datasets.} We use four multiple graphs datasets, namely PROTEINS, ENZYMES, IMDB-BINARY and IMDB-MULTI~\citep{yanardag2015deep,KKMMN2016}. We are interested in evaluating whole-graph representations under two different scenarios, one where the nodes have high-dimensional feature vectors and the other where the nodes do not have features. To this end, we chose the two biological networks PROTEINS and ENZYMES, where nodes contain feature vectors of size 32 and 21 respectively and the social networks IMDB-BINARY and IMDB-MULTI where nodes do not have features. More details in Section D of this \Appendix.

\textbf{Training the model.} Since we have multiple graphs in our datasets, our set of positive graph examples is already given in the data, unlike in \Cref{sec:exp}, where we had to subsample positives from a single graph. 
The negative examples still need to be sampled.
For the biological networks, we used the same negative sampling approach used in \Cref{sec:exp}. %
For the social networks, where the nodes do not have features, for each positive example, we uniformly at random add $n$ edges to it, generating a negative sample  (where $n$ is the number of nodes in the graph). %

\textbf{Experimental setup.} 
We equally divide the graphs in each dataset between training (unsupervised) and training+testing (supervised). 
We use two thirds of the graphs in the supervised dataset to train a logistic classifier for the downstream task over the graph's representation. 
We use a third of the supervised dataset to test the method's accuracy. 
The classification tasks used here are the same as in ~\citet{borgwardt2005protein} and ~\citet{xu2018how}. 
Again, we set the representation dimension of both \method and our baselines to 128. 
We show results for $k=3,4,5$ motifs representations, $k=n$ whole-graph representations, and unsupervised GNN node representations. 
To create these representations, we tested both sum and mean pooling for \method (except $k=n$) and all the node-based baselines. We report the best performance of each for a fair comparison.

\begin{table}[ht]
\caption{Results for the whole-graph classification task evaluated over balanced accuracy. We report mean and standrad deviation over five runs.}
	\centering
	\scalebox{.75}{		\begin{tabular}{lrrrr}
			\textbf{Method} & \textbf{PROTEINS} & \textbf{ENZYMES} & \textbf{IMDB-BIN.} & \textbf{IMDB-MULT} \\ \hline
			GS-mean$^\text{\citep{Hamilton2017}}$  &  0.753 $\pm$ 0.01 & 0.435 $\pm$ 0.02  & 0.454 $\pm$ 0.01  & 0.347 $\pm$ 0.01 \\
			GS-max$^\text{\citep{Hamilton2017}}$   & 0.729 $\pm$ 0.01 &  0.400 $\pm$ 0.04 &  0.447 $\pm$ 0.01 & 0.360 $\pm$ 0.01 \\
			GS-lstm$^\text{\citep{Hamilton2017}}$  &   0.739 $\pm$ 0.01 & 0.404 $\pm$ 0.04 & 0.442 $\pm$ 0.00 & 0.342 $\pm$ 0.01 \\
			DGI (Nodes)$^\text{\citep{velickovic2018graph}}$	&  0.743 $\pm$ 0.02 & 0.349 $\pm$ 0.04 &  $ 0.469 \pm$ 0.00 & 0.367 $\pm$ 0.02 \\
			DGI (Joint)$^\text{\citep{velickovic2018graph}}$ 	&  0.756 $\pm$ 0.00 & 0.263  $\pm$ 0.03  &  0.568 $\pm$ 0.03 & 0.376 $\pm$ 0.01 \\
			Raw Features  &  0.665 $\pm$ 0.05 & 0.210 $\pm$ 0.02  & --   &  -- \\
			NetLSD$^\text{\citep{tsitsulin2018netlsd}}$    & 0.760 $\pm$ 0.00	& 0.250 $\pm$ 0.00 & 0.550 $\pm$ 0.00 & 0.430 $\pm$ 0.01\\
			graph2vec$^\text{\citep{narayanan2017graph2vec}}$   & 0.685  $\pm$ 0.00	& 0.166 $\pm$ 0.00 & 0.507 $\pm$ 0.00  & 0.335 $\pm$ 0.00 \\
			InfoGraph$^\text{\citep{sun2020}}$   & 0.690 $\pm$ 0.04 	& 0.278 $\pm$ 0.04  & \textbf{0.691} $\pm$ 0.04   & \textbf{0.466} $\pm$ 0.02 \\
			\cmidrule(lr){1-1}
			\method (Rnd) ($k=3$)    %
			& 0.733 $\pm$ 0.01 & 0.293 $\pm$ 0.02	 & 0.586 $\pm$ 0.00 & 0.369 $\pm$ 0.001\\
			\method    ($k=3$)          &\textbf{0.777} $\pm$ 0.01	& \textbf{0.445} $\pm$ 0.01 & 0.586 $\pm$ 0.00 & 0.376 $\pm$ 0.00 \\
			\cmidrule(lr){1-1}
			\method (Rnd) ($k$=4)    %
			&  0.720 $\pm$ 0.02 & 0.229 $\pm$ 0.04	 & 0.580 $\pm$ 0.00 & 0.371 $\pm$ 0.00 \\
			\method   ($k=4$)          & \textbf{0.780} $\pm$ 0.02	&  0.390 $\pm$ 0.04 & 0.621 $\pm$ 0.00 & 0.390 $\pm$ 0.002 \\
			\cmidrule(lr){1-1}
			\method (Rnd) ($k=5$)    %
			& 0.722  $\pm$ 0.01 &  0.213 $\pm$   0.03	 & 0.580 $\pm$ 0.00 & 0.378 $\pm$ 0.005 \\
			\method    ($k=5$)         & \textbf{0.773} $\pm$ 0.01	&  0.326 $\pm$ 0.04 & 0.600 $\pm$ 0.01 & 0.397 $\pm$ 0.001 \\			
			\cmidrule(lr){1-1}
			\method (Rnd)  ($k=n$)    %
			& 0.704 $\pm$ 0.03 & 0.266 $\pm$ 0.02 & \textbf{0.707} $\pm$ 0.02 & \textbf{0.446} $\pm$ 0.005 \\
			\method       ($k=n$)      & 0.753 $\pm$ 0.00	& 0.327 $\pm$ 0.01 & \textbf{0.694} $\pm$ 0.02 & \textbf{0.451} $\pm$ 0.01\\
		\end{tabular}
		\label{tab:graphtasks}
	}
	\end{table}

\textbf{Baselines.} We compare \method against \emph{non-compositional methods}: pooling node representations from GraphSAGE and DGI, directly pooling node features, two recent whole-graph embedding methods, NetLSD~\citep{tsitsulin2018netlsd} and graph2vec~\citep{narayanan2017graph2vec} and a recent unsuperved whole-graph representation, InfoGraph~\citep{sun2020}. Apart from pooling node features, all methods input graph-wide information to their representations. Pooling node features is not applicable to the social networks, since they do not have such information. Additionally, DGI also generates a whole-graph representation to minimize the mutual entropy with the nodes' representations. Note how by setting $k=n$, we consider the entire graph as a single motif and thus, learn a whole-graph representation. Again, all models were trained according to their original implementation.

\vspace{0.15in}
\textbf{Results.} We show in \Cref{tab:graphtasks} the results for whole-graph classification downstream tasks. For each task and each model, we report the mean and the standard deviation of the balanced accuracy (mean recall of each class) achieved by logistic regression over five different runs. We observe how our method consistently outperforms representations computed over the entire graph: the joint DGI approach, graph2vec and node representations pooling. Interestingly, we observe that when the graph has high-dimensional feature vectors of the nodes, pooling small motif representations better generalizes than all other methods to unseen graphs. On the other hand, we observe that using a joint whole-graph representation, either with $k=n$ in our model or with NetLSD or with InfoGRAPH, can perform better without node features. In fact, there is no significant difference between using a random and a trained model for the joint representation. It is known how a random GNN model simply assigns a unique representation to each class of graphs indistinguishable under the 1-WL test~\citep{xu2018how}. Therefore, for graphs without node features, assigning unique representations seems to be the best in this setting, which means that the tested graph embedding and unsupervised representation methods are not really capturing significant graph information. Overall, we observe that indeed motif representations are capable of representing the entire graph to which they belong and even give better results, evidencing how \method is learning graph-wide information, \textit{i.e.} capturing $\pr(\bA,\mX;\bW)$ and how motif compositionality can explain networks functionality.

{\bf \method architecture.} 
We use the same $\rho$ and $\text{READOUT}$ functions as in \Cref{sec:exp}, while changing the GNN to GIN~\citet{xu2018how} (which gave better validation results than the GAT, GCN, and GraphSAGE GNNs).
Again, we use $M=1$, \textit{i.e.}, we sample one negative example for each positive sample. We show results of \method for $k=3,4,5,n$. 
For the estimator $\widehat{\Phi}(\bA,\mX;\bW)$, we perform 30 tours for every model and dataset. 

\textbf{GNN layer.} We use a single-layer GIN~\citet{xu2018how} as the GNN layer in our method. For $k=n$, where the GNN is applied over large graphs, we used GIN with two layers. Note that we also tested GraphSAGE-mean, GCN and GAT GNN layers here, but GIN resulted in faster training loss convergence.

\textbf{Number of tours.} We did 30 tours for all datasets. Again, we tested training models, each with a different fix number of tours, starting with 1 tour and increasing 10 by 10 until we reached the reported number of tours, which results in training loss convergence.

\textbf{Supernode size.} We did a BFS with the maximum number of subgraphs visited as 5K for all models (and all $k$). Again, we started with a small supernode budget of 100 and increased it by 100 until we observed the tours being completed and the training loss converging.

\textbf{Minibatch size.} We used a minibatch size of 50 for ENZYMES and PROTEINS for all reported $k$. For IMDB-BINARY and IMDB-MULTI, which have larger networks we used a minibatch size of 10. Again, we tested small minibatch sizes and increased until we had training loss convergence and GPU memory to use.

\textbf{Pooling functions.} We tested both sum and mean pooling motif (our model) and node (baselines) representations for all models here. We observed that mean pooling performs the best for all models in all datasets, except for the ENZYMES dataset, where sum pooling performed the best for all models. Thus, \Cref{tab:graphtasks} contain results with mean pooling for all models in the PROTEINS, IMDB-BINANRY and IMDB-MULTI datasets and sum pooling for all models in the ENZYMES dataset.

\section{Datasets}

We present the datasets statistics in \Cref{tab:citation} and \Cref{tab:graphs}. For the PROTEINS and ENZYMES datasets, we added the node labels as part of the node features. For the DBLP, we subsampled (with Forest Fire) the original large network from~\citet{yadati2019hypergcn}. For the Steam graphs, we consider user-product relations from 2014 to create the training graph and data from 2015 to create the test graph. Similarly, we use 2016 data to create the Rent the Runway training graph and 2017 data to create the test graph. For both product networks, the node features we created are sparse bag-of-words from the user text reviews.

\bgroup
\def\arraystretch{1.5}
\begin{table}[ht]
	\centering
	\caption{Single graph datasets statistics.}
	\scalebox{0.75}{
		\begin{tabular}{llrrr}
			\textbf{Dataset}  & \textbf{Type}  & \textbf{Nodes}  & \textbf{Edges}  & \textbf{Features} \\ \hline
			Cora~\citep{sen2008collective}    & Citation Network & 2,708  & 5,429  & 1,433   \\
			Citeseer~\citep{sen2008collective} & Citation Network & 3,327  & 4,732  & 3,703    \\
			Pubmed~\citep{sen2008collective}   & Citation Network & 19,717 & 44,338 & 500    \\
			DBLP~\citep{yadati2019hypergcn}   & Coauthorship Network & 4,309 & 12,863 & 1,425 \\
			Steam~\citep{pathak2017generating} (Train)  & Product Network & 1,098 & 7,839 & 775 \\
			Steam~\citep{pathak2017generating} (Test)  & Product Network & 1,322 & 7,547 & 775
			\\
			Rent the Runway~\citep{misra2018decomposing} (Train)  & Product Network & 2,985  & 55,979  & 1,475 \\
			Rent the Runway~\citep{misra2018decomposing} (Test)  & Product Network & 5,003  & 67,365 & 1,475
		\end{tabular}
	}
	\vspace{0.2in}
	\label{tab:citation}
\end{table}

\bgroup
\def\arraystretch{1.5}
\begin{table}[ht]
	\centering
	\caption{Multiple graphs datasets statistics.}
	\scalebox{0.75}{
		\begin{tabular}{llrrr}
			\textbf{Dataset}  & \textbf{Type}  & \textbf{Graphs}  & \textbf{Features} & \textbf{Classes} \\ \hline
			PROTEINS~\citep{KKMMN2016}     & Biological Network & 1,113  & 32  & 2   \\
			ENZYMES~\citep{KKMMN2016} & Biological Network & 600  & 21  & 6   \\
			IMDB-BINARY~\citep{KKMMN2016}   & Social Network & 1,000 & 0 & 2   \\
			IMDB-MULTI~\citep{KKMMN2016}   & Social Network & 1,500 & 0 & 3			
		\end{tabular}
	}
	\vspace{-0.1in}
	\label{tab:graphs}
\end{table}

\section{Related Work: Higher-order Graph Representations}\label{sec:rw2}

In what follows, we review the existing approaches to higher-order graph representations in literature.

\textbf{Higher-order graph representations.} Morris \textit{et. al}~\citep{morris2019weisfeiler} showed how to expand the concept of a GNN, an approach based on the 1-WL algorithm~\citep{weisfeiler1968reduction}, to a $k$-GNN, an approach based on the class of $k$-WL~\citep{cai1992optimal} algorithms, where instead of generating node representations, one can derive higher-order ($k$-size) representations later used to represent the entire graph. Although such approaches to represent entire graphs have been recently used in \textit{supervised} graph classification tasks, how to systematically use them in an {inductive unsupervised} manner was not clear. Since edge-based models require factorizing over a 2-node representation, %
only 1-WL~\citep{kipf2017semi,xu2018how,Hamilton2017,velivckovic2017graph} and 2-WL~\citep{morris2019weisfeiler}-based GNNs can be used. Additionally, $k$-GNNs can be thought of as a GNN over an extended graph, where nodes are $k$-node tuples and edges exist between $k$-tuples that share exactly $k-1$ nodes. One could indeed think of applying an edge-based loss to the extended graph, where the nodes ($k$-node tuples) representations are given by a $k$-GNN. However, an edge-based model assumes independence among edges and an edge in the extended graph is repeated several times in the extended graphs, thus they are not independent. Finally, even if one could provide an unsupervised objective to $k$-GNNs, it would still require $\mathcal{O}(n^k(k \delta)L)$ steps to compute an $L$-layer $k$-GNN over a graph with $n$ nodes and maximum degree $\delta$. Due to the non-linearities in the READOUT function and in the neighborhood aggregations in $k$-GNNs, unbiased subgraph estimators such as the one presented in this work and neighborhood sampling technique such as the one from \citet{Hamilton2017} would not provide an unbiased or a bounded loss estimation such as \method does. Moreover, the more recent sparser version of $k$-GNNs~\citep{morris2020} uses $k$-node tuple representations, instead of $k$-node subgraph represenations as in the original paper. Finally, \method can take advantage of any graph representation method, including $k$-GNNs~\citep{morris2019weisfeiler} and non-GNN approaches such the ones presented in Relational Pooling~\citep{murphy2019relational}. %

\textbf{Sum-based subgraph representations.} There has been recent work representing subgraphs by equating them with sets of node representations~\cite{HamiltonSurvey}. In general, these approaches use graph models able to generate node representations and then add a module on top to aggregate these individual representations in the downstream task. The most prominent efforts have treated subgraph representations as sums of the individual nodes' representations~\cite{HamiltonSurvey}, namely sum-based techniques. These approaches do not rely on joint subgraph representations, \textit{i.e.} subgraphs that share nodes will tend to have similar representations, constraining their representational power and thus relying more on the downstream task model.

\textbf{Hypergraph models.}
In this work, we wish to learn a graph model through motif representations in the presence of standard dyadic (graph) data, \textit{i.e.} we are only observing pairwise relationships. Therefore, we emphasize that hypergraph models, despite dealing with higher-order representations of graphs, require observing polyadic (hypergraph) data and therefore are not an alternative to the problem studied here. 

\textbf{Supervised learning with subgraphs.} \citet{jason2018SPNN} made the first effort towards supervised learning  with subgraphs, where the authors predict higher-order properties from temporal dyadic data, as opposed to the problem presented here, where we are are interested in \textit{inductive unsupervised} learning of $k$-node sets from static graphs. Moreover, Meng \textit{et. al} learned subgraph properties while optimizing a pseudo-likelihood function, \textit{i.e.} ignoring the dependencies among different subgraphs in the loss function. Because different node sets share edge variables, it is vital to learn dependencies among them. Hence, here we presented the first graph model based on $k$-size motif structures trained with a proper Noise-Contrastive Estimation function, \textit{i.e.} our model accounts for dependencies between every edge to represent $k$-size node sets. 

\end{document}